\documentclass{article}

\title{Active Bipartite Ranking with\\ Smooth Posterior Distributions}

\usepackage{times}
\usepackage[scale=0.70]{geometry}
\usepackage{graphicx} 
\usepackage{notations}
\usepackage{mathtools}
\usepackage{amsfonts, amsmath, amssymb}
\usepackage{amsthm}
\usepackage{microtype}
\usepackage{subfigure}
\usepackage{booktabs}  
\usepackage[utf8]{inputenc} 
\usepackage[T1]{fontenc}    
\usepackage{hyperref}       
\usepackage{url}            
\usepackage{xcolor}         
\usepackage[utf8]{inputenc} 
\usepackage[T1]{fontenc}    
\usepackage{url}            
\usepackage{nicefrac}       
\usepackage{xcolor}
\usepackage{natbib}
\usepackage{subcaption}
\usepackage{algorithmic}

\usepackage{algorithm}
\usepackage{wrapfig}
\usepackage{mathtools}

\usepackage[textsize=tiny]{todonotes}

\begin{document}
\title{Active Bipartite Ranking with\\ Smooth Posterior Distributions}
\date{}

\author{%
  James Cheshire \\
  Telecom Paris\\
  \texttt{firstname.lastname@telecom-paris.fr} \\
  \and
  Stephan Clémençon \\
  Telecom Paris \\
  \texttt{\texttt{firstname.lastname@telecom-paris.fr}} \\
}

\newtheorem{lemma}{Lemma}
\newtheorem{theorem}{Theorem}
\newtheorem{corollary}{Corollary}
\newtheorem{property}{Property}
\newtheorem{proposition}{Proposition}
\newtheorem{definition}{Definition}
\newtheorem{assumption}{Assumption}
\newtheorem{remark}{Remark}
\newtheorem{example}{Example}
\bibliographystyle{apalike}

\maketitle

\begin{abstract}

In this article, bipartite ranking, a statistical learning problem involved in many applications and widely studied in the passive context, is approached in a much more general \textit{active setting} than the discrete one previously considered in the literature. While the latter assumes that the conditional distribution is piece wise constant, the framework we develop permits in contrast to deal with continuous conditional distributions, provided that they fulfill a Hölder smoothness constraint. We first show that a naive approach based on discretisation at a uniform level, fixed \textit{a priori} and consisting in applying next the active strategy designed for the discrete setting generally fails. Instead, we propose a novel algorithm, referred to as \klcrank and designed for the continuous setting, which aims to minimise the distance between the ROC curve of the estimated ranking rule and the optimal one w.r.t. the $\sup$ norm. We show that, for a fixed confidence level $\epsilon>0$ and probability $\delta\in (0,1)$, \klcrank is PAC$(\epsilon,\delta)$. In addition, we provide a problem dependent upper bound on the expected sampling time of \klcrank and establish a problem dependent lower bound on the expected sampling time of any PAC$(\epsilon,\delta)$ algorithm. Beyond the theoretical analysis carried out, numerical results are presented, providing solid empirical evidence of the performance of the algorithm proposed, which compares favorably with alternative approaches.

\end{abstract}
\section{Introduction}
\vspace{-0.2cm}Whether in the medical field (diagnostic assistance), in signal processing (anomaly detection), in finance (credit-risk screening) or in automatic document retrieval (search engines), the goal pursued is not always to learn how to predict a binary label $Y$, valued in $\{0,1\}$ say (\textit{e.g.} ill \textit{vs} healthy, abnormal \textit{vs} normal, default \textit{vs} repayment, irrelevant \textit{vs} relevant) based on some related random input information $X$ like in binary classification, the flagship problem in machine-learning. More often, the aim is to learn a (real valued) ranking function $f(x)$ so as to order all possible values $x$ for $X$ like any increasing transform of the posterior probability $\eta(x)=\mathbb{P}(Y=1\mid X=x)$ would do. This statistical learning problem is known as \textit{Bipartite Ranking} and, due to the wide range of its applications, has received much attention these last few years. Because of its global (and not local) nature, the gold standard to measure ranking performance is of functional nature, namely the $P$-$P$ plot referred to as the $\roc$ curve, \textit{i.e.} the plot of the true positive rate $\mathbb{P}(f(X)>t\mid Y=1)$ against the false negative rate $\mathbb{P}(f(X)>t\mid Y=0)$ as the threshold $t\in \mathbb{R}$ varies.  Most of the works documented in the literature consider the passive batch learning situation, where the ranking function $f(x)$ is first built based on the preliminary observation of $n\geq 1$ independent copies of the generic random pair $(X,Y)$ in the training stage and next applied to new (temporarily) unlabeled observations in the test/predictive phase. Various algorithmic approaches and dedicated theoretical guarantees have been elaborated, based on direct optimization of the empirical $\roc$ curve (see \cite{CV09ieee} or \cite{CV08} for instance) or on maximization of scalar summary criteria such as the $\auc$ (\textit{i.e.} the empirical Area Under the $\roc$ Curve), refer to \textit{e.g.} \cite{AGHHPR05}, \cite{Clemencon08Ranking} or \cite{MW16}. Various extensions of the original framework have been recently studied: \textit{Multipartite Ranking}, which corresponds to the case where the ordinal label $Y$ takes a (finite) number of values strictly larger than $2$, is considered in \cite{CRV13}, the situation, referred to as as \textit{Continuous Ranking}), where $Y$ is a continuous real valued r.v. is investigated in \cite{NIPS2017_97416ac0} while the case where no label is available, known as \textit{Unsupervised Ranking}, is investigated in \cite{10.1214/18-EJS1474}. 

In bipartite ranking, the statistical framework is as follows. One observes $n\geq 1$ independent copies $\{(X_1,Y_1),\; \ldots,\; (X_n,Y_n)\}$ of a generic random pair $(X,Y) \in [0,1]\times\{0,1\}$. We assume $X$ is draw in uniformly from $[0,1]$ and models some information hopefully useful to predict $Y$. Specifically, $Y \sim \cD_X$, for some unknown distribution $ \cD_X$. In contrast to classification, the goal pursued is of global (and not local) nature. It is not to assign a label, positive or negative, to any new input observation $X$ but to rank any new set of (temporarily unlabeled) observations $X'_1,\; \ldots,\; X'_{n'}$ by means of a (measurable) scoring function $s:[0,1]\rightarrow \mathbb{R}$, so that those with greater labels appear on top of the list (\textit{i.e.} are those with the highest scores) with high probability. More formally, the accuracy of any scoring rule can be evaluated through the $\roc$ curve criterion or its popular scalar summary, the $\auc$ (standing for the Area Under the $\roc$ Curve), and, as expected, optimal scoring functions w.r.t. these performance measures can be shown to be increasing transforms of the posterior probability $\eta(x) =\mathbb{P}(Y=1 \mid X=x)$. Though easy to formulate, this problem encompasses many applications, ranging from credit risk screening to the design of decision support tools for medical diagnosis through (supervised) anomaly detection. The vast majority of dedicated articles consider the \textit{batch} situation solely, where the learning procedure fully relies on a set of training examples given in advance, however, more recently in \cite{cheshire2023active}, bipartite ranking has been considered in an \textit{active learning} framework. In such an active setting the learner is able to formulate queries in a sequential manner, so as to observe the labels at new data points in order to refine progressively the scoring/ranking model. In the work of \cite{cheshire2023active} a key (restrictive) assumption is made that the regression function $\eta$ is piece wise constant on a grid of \textit{known} size $K$ on the feature space. In one dimension, this means the learner assumes there are some $(\mu_k)_{k\in [K]} \in [0,1]^K$ such that: 

\begin{equation}\label{eq:piececont}
\eta(x) = \sum_{i=1}^K \mu_k \Ind\big(x \in [i/K,(i+1)/K)\big)\;.
\end{equation}

where $\Ind$ denotes the indicator function. As pointed out in \cite{cheshire2023active}, we see that under the assumption of Equation \eqref{eq:piececont}, the problem becomes equivalent to a \textit{multi armed bandit}. In the multi armed bandit problem, the learner is presented with a set of $K$ "arms", each one corresponding to some unknown distribution. The learner then plays a game of several rounds, where in each round the learner must choose one of the $K$ arms from which they receive a sample from the corresponding distribution, see \cite{lattimore2020bandit} for an overview. Much of the literature on multi armed bandits concerns maximisation of cumulative reward, however, several strands of literature consider so called "pure exploration" bandit problems, where the learner does not care about there cumulative reward, but rather aims to uncover some underlying property of the arms. The classical example of a pure exploration bandit problem is best arm identification. We see that, as $K$, the size of the grid upon which $\eta$ is assumed to be piecewise constant, is taken as known to the learner, the problem under assumption of Equation \eqref{eq:piececont} is then equivalent to a $K$ armed pure exploration bandit problem - one can view each piece wise constant section of the grid as an arm. In this setup, as studied in \cite{cheshire2023active}, the goal of the learner is to uncover the ranking of the arms, with their regret given by the ROC criterion.

A fundamental problem, recognised in the bandit literature, is that, assuming one has a finite set of arms, which can all be sampled in reasonable time, is unsuitable in practice. There have been various responses to this problem, one being continuous armed bandits, otherwise known as $\cX$-armed bandits. Here, instead of a finite set of arms, one considers a function on some feature space $\cX$. The actions of the learner now correspond to querying points of the feature space and receiving noisy evaluations of the function at said points. Many works have considered the classical bandit problems of minimising cumulative regret and best arm identification, in the $\cX$-armed bandit setting, see \cite{bubeck2011x}, \cite{grill2015black},\cite{pmlr-v98-bartlett19a}, \cite{akhavan2020exploiting} and \cite{locatelli2018adaptivity}. As the above works have considered classical bandit problems in the $\cX$-armed setting, our goal in this paper is to consider the problem of bipartite ranking in a continuous setting. That is, we remove the piece wise constant assumption on $\eta$, which is rather assumed to be a continuous function, subject to certain smoothness constraints, as in the classic $\cX$-armed bandit setting. 

\paragraph{Our contributions} In this paper we consider the bipartite ranking problem under a smoothness constraint on the posterior $\eta$. Specifically, we remove the piece wise constant assumption of \cite{cheshire2023active} and instead assume that the regression function $\eta$ is $\beta$-Hölder smooth in each dimension, see Assumption \ref{ass:1}, where \textit{$\beta$ is known to the learner}. We describe an algorithm, \klcrank, which sequentially queries points of the feature space $\cX$. Given a confidence level $\epsilon>0$ and probability $\delta>0$, the goal of \klcrank is, in a few queries as possible, to output a ranking of $\cX$, such that the induced ROC curve of said ranking is within $\epsilon$ of the optimal ROC curve, in terms of the sup norm, with probability greater than $1-\delta$. Theorem \ref{thm:packlcrank} then shows that \klcrank satisfies the above statistical guarantee and provides an upper bound on it's expected total number of queries. In Theorem \ref{thm:lb}, we also demonstrate a lower bound on the expected number of queries of any PAC$(\epsilon,\delta)$ algorithm, matching the upper bound of Theorem \ref{thm:packlcrank}, up to log terms.

\section{Background and Preliminaries}\label{sec:background}

\paragraph{Notation}\label{sec:not} Here we introduce several dedicated notions that will be extensively used in the subsequent analysis. By $\lambda$ is meant the Lebesgue measure on $[0,1]^d$.  For any $a,\; b$ in $[0,1]$, $\Ber(a)$ refers to the Bernoulli distribution with mean $a$ and $\kl(a,b)$ to the Kullback Leibler divergence between
 the Bernoulli distributions $\Ber(a)$ and $\Ber(b)$. The indicator function
 of any event $\xi$ is denoted by $\Ind\left(\xi\right)$, for $d\in \mathbb{N}, \epsilon>0$, we write $\cE_d(\epsilon)$ for the $\epsilon$-net on $[0,1]^d$ and for $x \in [0,1]^d$, $r>0$, $\cB_x^d (r)$ for the $d$-dimensional ball of radius $r$ centered at $x$. 
 
\paragraph{Bipartite ranking}
A rigorous formulation of ranking involves functional performance measures. Let $\cS$ be the set of all scoring functions, any $s\in \cS$ defines a preorder $\preceq_s$ on $\cX$: for all $x,x'\in \cX$, $x\preceq_s x' \Leftrightarrow s(x)\leq s(x')$. From a quantitative perspective, the accuracy of any scoring rule, can be evaluated through the $\roc$ curve criterion, namely the PP-plot $t\in \mathbb{R}\mapsto (1-H_{s}(t),\; 1-G_{s}(t))$, where
$H_{s}(t)=\mathbb{P}\left\{ s(x)\leq t \mid Y =  0  \right\}$ and  $G_{s}(t)=\mathbb{P}\left\{ s(x)\leq t \mid Y = 1\right\}$, for all $t\in \mathbb{R}$. The curve can also be viewed as the graph of the c\`ad-l\`ag function $\alpha\in (0,1)\mapsto \roc (s,\alpha)=1-G_s\circ H^{-1}_{s}(1-\alpha)$. The notion of $\roc$ curve defines a partial order on the set of all scoring functions (respectively, the set of all preorders on $\cX$): $s_1$ is more accurate than $s_2$ when $\roc(s_2,\alpha)\leq \roc(s_1,\alpha)$ for all $\alpha\in (0,1)$.  As can be proved by a straightforward Neyman-Pearson argument, the set $\mathcal{S}^*$ of optimal scoring functions is composed of increasing transforms of the posterior probability $\eta(x)=\mathbb{P}\{Y=1 \mid X=x  \}$, $x\in \cX$. We have 
$$\mathcal{S}^*=\{s\in \cS: \; \;\forall x,\; x' \in[0,1]^d,\;\; \eta(x)<\eta(x') \Rightarrow s^*(x)<s^*(x') \}\;,$$ 
and 
$$\forall (s,s^*)\in \cS\times \cS^*,\; \forall \alpha\in (0,1),\; \roc(s,\alpha)\leq \roc^*(\alpha):=\roc(s^*,\alpha)\;.$$ 
The ranking performance of a candidate $s\in \cS$ can be thus measured by the distance in $\sup$-norm between its $\roc$ curve and $\roc^*$, namely $$d_{\infty}(s,s^*) := \sup_{\alpha \in (0,1)}\{ \roc^*(\alpha) - \roc(s,\alpha)\}\:.$$
An alternative convention to represent the ROC of a scoring function $s$, which we will use for the remainder of this paper, is to consider the broken line $\widetilde{\roc}(s,.)$, which arises from connecting the PP-plot by line segments at each possible jump of the cdf $H_s$. {\bf From here on out when referring to the $\roc$ of a scoring function $s$, we refer to the broken line $\widetilde{\roc}(s,.)$.}

\paragraph{The active learning setting}
Whereas in the batch mode, the construction of a nearly optimal scoring function (\textit{i.e.} a function $s\in \cS$ such that $d_{\infty}(s,s^*)$ is 'small' with high probability) is based on a collection of independent training examples given in advance, the objective of an \textit{active learner} is to formulate queries in order to recover sequentially the optimal preorder on the feature space $\cX$ defined by the supposedly unknown function $\eta$. That is, the active learner plays a game with multiple time steps, where, at time each step $t>0$, they must choose a point $a_t \in \cX$ to query, so as to observe the random label $Y_t \sim \cB er\left(\eta(a_t)\right)$ and refine the scoring model incrementally. After a sufficient number of rounds has elapsed, chosen at the learner's discretion, a final scoring function $\heta$, is output. We refer to the regret of the learner as the distance in the sup-norm between the ROC curves of $\heta$ and $\eta$, that is the regret of the learner is given as, $d_\infty(\eta,\heta)$.
\paragraph{Assumptions on the feature space and posterior}
We assume the feature space $\cX$ is of the form $[0,1]^d$ for some dimension $d\in \mathbb{N}$. We assume that the regression function, i.e. the posterior probability $\eta$ is $\beta$-Holder smooth in each dimension, see Assumption \ref{ass:1}.
\begin{assumption}\label{ass:1}
For some $\beta > 0$, known to the learner, there exists a constant $C> 0$ such that $\forall, x,y\in [0,1]^d$,

$$\forall i \leq d, |\eta(x_i) - \eta(y_i)| \leq C|x_i - y_i|^\beta \;.$$ 
\end{assumption}

For clarity we assume $C=1$, otherwise our results and algorithms would change only in the constant terms, which would then depend upon $C$. We write $\cB$ for the set of all problems, where each problem $\nu \in \cB$ is defined by a specific posterior probability $\eta : \cX \rightarrow [0,1]$ satisfying Assumption \ref{ass:1}. For simplicity, we suppress the dependency upon the smoothness constant $\beta$ and dimension $d$ in the notation $\cB$. 

\paragraph{Policies and fixed confidence regime.} We denote the outputted scoring function of the learner $\hat \eta \in S$. The way the learner interacts with the environment - i.e. their choice of points to query, how many samples to draw in total and their choice of $\hat \eta \in S$, we term the \textit{policy} of the learner. 
We write $\cC$ for the set of all possible policies of the learner. For a policy $\pi \in \cC$ and problem $\nu \in \cB$ we denote random variable $\tau_\nu^\pi$ as the stopping time of policy $\pi$. We write $\heta_\nu^\pi$ for the scoring function outputted by policy $\pi$ on problem $\nu$. Where obvious we may drop the dependency on $\pi,\nu$ in the notation, referring to the scoring function outputted by the learner as simply $\heta$.
We write $\mathbb{P}_{\nu,\pi} $ as the distribution on all samples gathered by a policy $\pi$ on problem $\nu$. We similarly define $\mathbb{E}_{\nu,\pi}$. For the duration of this paper we will work in the \textit{fixed confidence regime}. For a confidence level $\epsilon>0$ and $\delta>0$, a policy $\pi$ is said to be PAC$(\delta,\epsilon)$ (probably approximately correct), on the class of problems $\cB$, if, 
\begin{equation}\label{eq:pac}
\forall \nu \in \cB, \; \mathbb{P}_{\nu,\pi} \left(d_\infty(\heta,\eta)\leq \epsilon\right) \geq 1- \delta\;.    
\end{equation} The goal of the learner is to then obtain a PAC$(\delta,\epsilon)$ policy $\pi$, such that the expected stopping time in the worst case, $\sup_{\nu \in \cB} \mathbb{E}_{\nu,\pi}[\tau_\nu^\pi]$, is minimised. 

\paragraph{Problem complexity}
With the term "problem complexity", we refer to the minimum number of samples the learner can expect to draw, while remaining PAC$(\epsilon,\delta)$. It is a quantity that will depend upon the features the problem, i.e. the shape of the posterior $\eta$. Our problem complexity will follow from \cite{cheshire2023active} where the authors consider the problem under the assumption $\eta$ is piece wise constant on a uniform grid of known size $K$, where the sequence of means $\mu_1,...\mu_K$ denote the value of $\eta$ on each grid point, see Equation \eqref{eq:piececont}. In \cite{cheshire2023active}, the authors demonstrate that, for a grid point $i \in [K]$, a PAC$(\epsilon,\delta)$ learner must be able to correctly distinguish $\mu_i > \mu_j$ vs $\mu_j > \mu_i$ for all $j \in [K]: |\mu_i - \mu_k| \leq \Delta_i$, where, 
$$\Delta_i := \min\bigg\{z > 0 : z\sum_{i\neq j}  \Ind\big(|\mu_i - \mu_j| \leq x \big) \geq K\epsilon p (1 - \mu_i) \bigg\}\;.$$ 
Furthermore they show that this gap $\Delta_i$ is tight, in the sense that if the learner cannot correctly distinguish $\mu_i > \mu_j$ vs $\mu_j > \mu_i$ for all pairs $i,j : |\mu_j - \mu_i| \geq \Delta_i$, they are not PAC$(\delta,\epsilon)$, see Lemmas 2.2 and 2.3 in \cite{cheshire2023active}. We adapt the discrete gap $\Delta_i$ to our continuous setting as follows, for a point $x\in\cX$ define the following gap $\Delta(x)$,
\begin{equation}
\Delta(x) := \min\bigg\{z > 0 : z\lambda(\{y: |\eta(x) - \eta(y)| \leq z\} ) \geq \epsilon p (1 - \eta(x)) \bigg\}\wedge (1-\eta(x))\;, 
\end{equation}
where $p:= \int_{x\in [0,1]^d} \eta(x) \; dx$ and we remind the reader that $\lambda$ denotes the Lebesgue measure on $[0,1]^d$. The gap $\Delta(x)$ can be seen as the minimum radius of the ball around $x$, such that in the worst case, miss ranking all points in $\cB_{\Delta(x)}^d(x)$, in comparison to $x$, one suffers more than $\epsilon$ regret, i.e. $d_\infty(\heta,\eta) > \epsilon$. Naturally the gap $\Delta(x)$ decreases with the tolerance parameter $\epsilon$, the smaller our tolerance for error, the more precise we must be in our ranking. We note that $\Delta(x)$ also depends on several properties of the regression function $\eta$ in the neighborhood of $x$. Firstly, as the posterior probability at $x$ increases compared to the average posterior probability $p$, the gap $\Delta(x)$ will decrease. This follows our intuition, as points of the feature space with high posterior probability correspond to the start of the ROC curve, where the gradient is higher and small errors correspond to a greater regret. Secondly, when the set of points with similar posterior probability to $x$ is large, the gap $\Delta(x)$ will again decrease. This dynamic also follows our intuition, as misranking a large sections of the feature space, even if only by a small amount, can have a substantial effect on our regret. 

To correctly rank $x$ against all points $y: |\eta(x) - \eta(y)| \geq \Delta(x)$, with high probability, the learner will need an estimate of the posterior probability at $x$, with high probability and precision at least $\Delta(x)$. Referring to well known results from A/B testing, see \cite{kaufmann2014complexity} and the discussion in Section 2.2 of \cite{cheshire2023active}, we see that, to do this, the learner will need to draw at least roughly $\frac{1}{\kl(\eta(x) - \Delta(x), \eta(x) + \Delta(x))}$ samples from point $x$. Of course, the learner cannot sample all points in our continuous feature space, they must instead exploit the smoothness constraint of Assumption \ref{ass:1}, with this in mind, we now define the complexity of a point,

\begin{equation}\label{eq:samp}
H(x) := \frac{\Delta(x)^{-d/\beta}}{\kl(\eta(x) - \Delta(x), \eta(x) + \Delta(x))}\;,
\end{equation}

with the minimum number of samples needed by the learner then of the order, $\int_{x \in [0,1]^d} H(x)\;dx\;.$ Our sample complexity is seen to be well chosen, as it is attained by the upper and lower bounds of Theorems \ref{thm:packlcrank} and \ref{thm:lb} respectively, up to constant and log terms. 

\subsection{Related literature}\label{sec:lit}

\paragraph{Comparison to the discrete setting}
There is a wide range of literature concerning ranking in the discrete setting - that is, where one has a finite number of actions to rank. To the best of our knowledge the closest work to our own is that of \cite{cheshire2023active}, as they consider directly the discrete version of our setting. For completeness we will first mention other related strands of literature. Several works consider the problem of ranking $n$ experts based on there performance across $m$ tasks. Recently, in \cite{pilliat2024optimal} they consider this problem in the batch setting. Specifically they deal with the task of ranking the rows of a matrix with isotonic columns, from which the learner receives a batch of noisy observations. As pointed out in \cite{cheshire2023active}, algorithms designed for the batch setting perform poorly in active bipartite ranking. In \cite{pmlr-v202-saad23b} the authors tackle the problem of ranking experts in an active learning setting, wherein the learner makes sequential queries to pairs of actions and experts, receiving a noisy evaluation of an experts performance on said action. They make a monotonictiy assumption on the experts, that is for each pair of experts, one outperforms the other on all tasks. If one restricts to the setting $m=1$, their setting is very similar to that considered in \cite{cheshire2023active}. However, crucially, they only consider the case where the learner is required to return a perfect ranking with high probability. This fundamentally changes the problem when compared to \cite{cheshire2023active} and our own setting, as here the learner is only required to return an essentially "$\epsilon$ good ranking" - that is the the ROC curve of their estimated scoring function must be within epsilon of the optimal ROC curve, in the $d_\infty$ distance. 

Another stream of literature concerns pairwise comparisons. In the active setting, see e.g. \cite{jamieson2011active}, \cite{heckel2019active}. In \cite{heckel2019active}, each pair of actions $i,j\in [K]$ has a corresponding probability $M_{i,j}$, being the probability item $i$ beats item $j$. The correct ranking is then given by an ordering according to the Borda scores of the actions. Compared to our setting, differing strategies are required to deal with noisy pairwise comparisons. To estimate the Borda score of a particular action, one must sample many pairs. However, in our setting, one can estimate the score, i.e. the $\eta$ of an action by sampling it in isolation. Furthermore in \cite{heckel2019active}, the goal of the learner is to give a perfect ranking with high probability, which again distances their work from our own. In the \cite{NIPS2017_db98dc0d} the authors tackle the problem of finding a "$\epsilon$-good" Borda ranking of the arms, in the sense that the learner cannot incorrectly rank any two arms who's Borda scores differ by more than $\epsilon$.  In this setting, for a certain arm, one needs to ensure that either, it is well ranked against all others with high probability, or that the width of the confidence interval for the Borda score of said arm, is less than $\epsilon$. In our setting, on the other hand, the required confidence for a single point is more complex. Due to the global nature of the ROC curve, it does not suffice to simply have the level of confidence at $\epsilon$, across the feature space, but rather the confidence level must depend on local and global properties of the regression function, see Equation \eqref{eq:samp}. For similar reasons, considering a PAC$(\epsilon,\delta)$ setup with regards to either the maximum difference between missed ranked arms or the number of incorrectly ranked arms, as in \cite{inproceedings}, is of a very different nature to our work. 


\paragraph{Comparison to $\cX$ - armed bandits}
The $\cX$ armed bandit is well studied, particularly for the problems of optimisation and cumulative regret minimisation. As in our case, the majority of the literature makes some form of smoothness assumption on the regression function. For both optimisation and cumulative regret minimisation, it is natural to consider functions that are smooth around one of their global optima. This is in contrast to our setting, as ranking is a global problem, one needs a global smoothness constraint. In \cite{bubeck2011x} the authors propose an algorithm \texttt{HOO} which utilises a hierarchical partitioning of the feature space, by iteratively growing a binary tree. When choosing the next node to split, they follow an optimistic policy, choosing the node with the highest upper bound on its expected payoff. Naturally, the expected utility of a node differs fundamentally from the case of optimisation to that of ranking. Further more, such a tree based algorithm would be cumbersome in our setting, as the width of KL divergence based confidence intervals varies across the feature space. 
The algorithm \texttt{HOO} takes the smoothness parameters as known, in \cite{grill2015black}, and further developed in \cite{pmlr-v98-bartlett19a}, the authors consider the problem of optimisation and propose an algorithm \texttt{POO}, which takes \texttt{HOO} as a subroutine and runs across many smoothness parameters, this cross validation step is an attempt at dealing with unknown smoothness. In \cite{grill2015black} the authors also consider a none topological smoothness assumption, that relates directly to the hierarchical partition of the feature space. Essentially, they assume that in the cell of depth $h$ containing the optima, the regression function is close to the optima on some rate depending on $h$, this idea is further developed in \cite{pmlr-v98-bartlett19a}. It would be difficult to formulate our problem under such a hierarchical smoothness assumption. As the ranking problem is global, we would need to control the variance of the function in each cell, giving something similar to our Hölder smoothness assumption while being more cumbersome.

In addition to optimisation, the $\cX$ armed bandit has been studied under cumulative regret minimisation. In \cite{locatelli2018adaptivity} it is shown that optimal adaptation to the smoothness constraint is impossible when one wishes to minimise cumulative regret. The $\cX$ armed bandit has also seen interest for quantile and CVaR optimisation, see \cite{torossian2019x}. The $\cX$ armed bandit has also been considered in the noiseless setting, see \cite{pmlr-v70-malherbe17a}, \cite{kawaguchi2016global}. There are multiple works considering infinite armed bandit problems where there is no topological relation between the arm indices and means, i.e. when the learner draws arms from a reservoir, see \cite{chaudhuri2017pac},\cite{aziz2018pure},\citep{samuels2020complexity}, \cite{heide2021bandits}. The lack of such topological relation gives the above literature a very different flavour to our own. 

\paragraph{Novelty of our results in comparison to \cite{cheshire2023active}} Let us consider the performance of a naive adaptation of the \rankmessy algorithm of \cite{cheshire2023active} to our setting. The \rankmessy algorithm is designed to operate under the assumption that $\eta$ is piece wise constant on some uniform grid of known size $K$. Under our Hölder smoothness constraint on $\eta$, if one runs \rankmessy on a uniform grid of high enough level, i.e. set $K$ large enough, we can expect \rankmessy to be PAC$(\epsilon,\delta)$ in our setting. For this to be the case, roughly speaking, the discretisation error associated with our proposed adaptation must be at most $\min_{x\in [0,1]}\Delta(x)$. Consequently, for such a naive approach we would need to run the \rankmessy on the grid $\cE_d\left(\min_{x\in [0,1]^d}\Delta(x)^{1/\beta}\right)$, i.e. set $K = \min_{x\in [0,1]^d}\Delta(x)^{-d/\beta}$ which would then have roughly the following upper bound on expected sampling time, up to log terms,
\begin{equation}\label{eq:discrete}
\sum_{i \in \cE_d\left(\min\limits_{x\in [0,1]^d}\Delta(x)^{1/\beta}\right)} \max_{y\in [0,1]} \frac{1}{\kl(\eta(y) - \Delta_i, \eta(y) + \Delta_i)} \;,   
\end{equation}
where we let $\Delta_i$ be the maximum of $\Delta(x)$ for $x$ in the $ith$ section of the grid. The above bound can differ considerably from our sample complexity, Equation \eqref{eq:samp}, in the case where the gap $\Delta(x)$ varies across the feature space. In such cases, the level of discretisation may be far too high in certain parts of the feature space, resulting in the learner drawing an unnecessary number of samples. Such a cases are exactly the ones of interest, where one wishes to exploit the benefits of active learning. Also, to adapt the \rankmessy algorithm in such a way, one would need to have knowledge of $\min_{x\in [0,1]^d}\Delta(x)$, which is an unreasonable assumption in practice. For these two reasons, adapting the previous approach of \cite{cheshire2023active} to our setting is not suitable and a novel approach is required. Our proposed algorithm, \klcrank, is able to vary the level of discretisation across the interval, according to $\Delta(x)$, without requiring any knowledge of $\Delta(x)$ itself.  Furthermore, the theoretical bounds in \cite{cheshire2023active} fail to exploit the fact that KL divergence based confidence intervals change in width across the interval $[0,1]$, being tighter when closer to $0$ or 1. This problem is acknowledged and stated as an open question in \cite{cheshire2023active} and our Theorem \ref{thm:packlcrank} correctly captures this dependency. 

\section{Main Theoretical Results}

\subsection{The \klcrank algorithm}
We first establish some notation. We start at time $t=0$ and each time \klcrank draws a sample we progress to the next time step. The algorithm \klcrank exclusively samples points from the finite, active set of points $\cX_t \subset \cX$, where at each time step, there is the possibility of adding new points to $\cX_t$. At time $t$, for $i \in \cX_t$ we write $N_t (i)$ for the total number of samples drawn from point $i$ up to time $t$ and $\hat{\mu}^t_{i}$ for the empirical mean of point $i$, calculated from all samples drawn from point $i$ up to time $t$. For exploration parameter $\beta(t,i,\delta):\mathbb{N}
 \times [0,1]^d\times \mathbb{R_+} \rightarrow \mathbb{R_+}$, we then define the LCB and UCB index,
\begin{equation}\label{eq:lcb}
    \mathrm{LCB}(i,t) := \min\left\{q\in \left[0,\hat{\mu}^t_{i}\right]: \kl\left(\hat{\mu}^t_{i},q\right) \leq \frac{\beta(t,i,\delta)}{N_{t}(i)}\right\}\;,
\end{equation}

\begin{equation}\label{eq:ucb}
    \mathrm{UCB}(i,t) := \max\left\{q\in \left[\hat{\mu}^t_{i},1\right]: \kl\left(\hat{\mu}^t_{i},q\right) \leq \frac{\beta(t,i,\delta)}{N_{t}(i)}\right\}\;.
\end{equation}
For the empirical gap at time $t$ at point $i$, we write, $\hDelta_{i,t} := \ucb(i,t) - \lcb(i,t)$. For $t>0$, the algorithm \klcrank maintains an estimate $\hp_t$ of the proportion $p$, regularly updating $\hp_t$ by drawing additional samples uniformly from $[0,1]^d$, see line \ref{alg:p} of \klcrank. By slight abuse of notation, we write $N_t(0)$ for the number of samples used to generate the empirical mean $\hp_t$ at time $t$. We then define 
$$\mathrm{LCB}(0,t) := \min\left\{q\in \left[0,\hat{p}_t\right]: \kl\left(\hat{p}_t,q\right) \leq \frac{\beta(t,0,\delta)}{N_t(0)}\right\}\;,$$
as the lower confidence bound on our estimation of $p$ and similarly define $\mathrm{UCB}(0,t)$.
We write $\hDelta_{0,t} = \ucb(0,t) - \lcb(0,t)$ for the width and upper and lower bounds of the confidence interval centered at $\hp_t$. 

The algorithm \klcrank is an elimination algorithm, in that it maintains an active section of the feature space, denoted $S_t\subset [0,1]^d$ at time $t$, with $S_0 = [0,1]^d$ and sequentially eliminates sections of the active set, until time $t$ at which $S_t = \emptyset$. Sections are eliminated as follows, we maintain an active, finite, set of points $\cX_t\subset[0,1]^d$, and  \klcrank draws samples from the intersection of the sets $\cX_t \cap S_t$. At each time step we sample the point $\argmax_{i \in \cX_t \cap S_t}(\hDelta_{i,t})$, then, if a point $i$ in $\cX_t \cap S_t$ satisfies our elimination rule, see Equation \ref{eq:elim}, we remove it and the set of points in its neighborhood from $S_t$, specifically we remove the set, $\{x: \argmin_{j \in \cX_t}(   \lVert j - x \rVert_d)= i\}$ from $S_t$, see line \ref{alg:elim} of \klcrank. Note that we \textit{never remove points from} $\cX_t$, rather, once a point $i \in \cX_t$ does not exist in $S_t$, we no longer have the possibility of drawing samples from it. Our elimination rule is designed such that, we only eliminate sections of $S_t$, once we are sufficiently confident in their ranking. At all times we have a natural ordering on the points within $\cX_t$, according to their empirical means, we can then also provide a scoring function of the entire feature space,
$$\heta_t(x) = \hsigma_t\left(\argmin_{j \in \cX_t}(   ||j - x||_d)\right)\:\;,$$
where $\hsigma_t$ is the permutation sorting $(\hmu^t_i)_{i \in \cX_t}$ into ascending order. The regret incurred by such a scoring function, i.e. $d_\infty(\heta_t,\eta)$ can be controlled via our stochastic regret, that is the differences between the empirical and true means of points in $\cX_t$ and also our discretisation error, that is the error incurred by approximating the continuous function $\eta$ from a finite set of points $\cX_t$. For maximum efficiency we wish for our stochastic and discretisation errors to be roughly the same. We estimate our maximum stochastic error across our active set, at time $t$, via $\Delta_{(t)} := \max_{i\in\cX_t \cap S_t}(\hDelta_{i,t})$. We then ensure that $\cX_t$ provides a covering of level $\Delta_{(t)}^{1/\beta}$ on our active set $S_t$, adding new points to $\cX_t$ if necessary, see line \ref{alg:disc} of \klcrank. Thus via Assumption \ref{ass:1}, we maintain across all $t$, that our disctritisation error is of the same order as our estimated stochastic error. In this fashion, our algorithm \klcrank adds new points to $\cX_t$ across the running time, as $\Delta_{(t)}$ decreases with $t$. As a result, \klcrank will vary the final discretisation level across the feature space - the longer a section of the feature space remains in the active set $S_t$, the greater the ultimate level of disctretisation will be on that section. As the max size of $\cX_t$ is not fixed, the exploration parameter will need to grow with the size of $\cX_t$ we will set $\beta(t,i,\delta) = c\log(t^2\hDelta_{i,t}^{-d/\beta}/\delta)$, with $c>0$ an absolute constant. This is in contrast to the \rankmessy algorithm of \cite{cheshire2023active}, which does not have this problem as they have an upper bound $K$, on the number of points the algorithm will draw samples from.

\paragraph{Elimination rule}
Roughly speaking, we wish to design our elimination rule, such that, a point $i\in \cX_t$ satisfies our elimination rule when $\Delta_{(t)}\lesssim \Delta(i)$. To do this, we will use our estimates of the $\Delta(i)$, $\hat \Delta_{i,t}$, via their maximum $\Delta_{(t)}$ and define our elimination rule as follows. At time $t>0$ for $i \in \cX_t$, $z > 0$, define, $U_{i,t}(z) := \left\{j \in \cX_t\cap S_t :  |\hmu_i^t - \hmu_j^t| \leq z\right\}$ and then define the set of points to be eliminated at the end of round $t$ as,

\begin{equation}\label{eq:elim}
\cQ_t := \Bigg\{i \in \cX_t \cap S_t: \Delta_{(t)}\leq \left(\frac{\epsilon \hat p_t}{\Delta_{(t)}^{d/\beta}|U_{i,t}(6\Delta_{(t)})|}\wedge 1\right)  (1 - \hmu_{i}^t)\Bigg\}\;.   
\end{equation}

If a point $i$ is in $\cQ_t$, itself and its neighborhood, specifically the set $\{x: \argmin_{j \in \cX_t}(   \lVert j - x \rVert_d)= i\}$, is removed from the active set $S_t$, i.e. $S_{t+1} = S_t \setminus \bigcup_{i\in \cQ_t} \{x: \argmin_{j \in \cX_t}(   \lVert j - x \rVert_d)= i\}$.

\begin{algorithm}[H]
\caption{\klcrank}
\label{alg:klcrank}
\begin{algorithmic}[1]
\STATE {\bf Input:}  $\epsilon>0,\delta>0,\beta>0$
\STATE {\bf Initialise:} $S_0 = [0,1]^d$, $\cX_0 = \cE_d(0.5)$
\REPEAT
\IF{$\hDelta_{0,t} \geq \Delta_{(t)}$} 
\REPEAT
\STATE Update the empirical mean $\hp_t$ by $Y_t$ where $a_t$ is drawn uniformly from $[0,1]^d$ \label{alg:p}
\UNTIL{$\hDelta_{0,t} \leq \max_{i \in \cX_t \cap S_t}(\hDelta_{i,t})$}
\ELSE 
\STATE Sample point $\argmax_{i \in \cX_t \cap S_t}(\hDelta_{i,t})$ \label{alg:samp}
\ENDIF
\FOR{$i \in S_t \cap \cX_t$}
    \IF{$i \in \cQ_t,  \max_{i \in \cX_t \cap S_t}(\hDelta_{i,t}) \leq \hat p_t/4$}
  \STATE $S_t = S_t \setminus \{x: \argmin_{j \in \cX_t}(   \lVert j - x \rVert_d)= i\}$ \label{alg:elim}

     \ENDIF
\ENDFOR
\STATE Let $\tilde n_t = \min\left(n : 2^{-n} \leq \Delta_{(t)}^{1/\beta}\right)$. Add the points $\cE_d(2^{-\tilde n_t }) \cap S_t \setminus \cX_t$ to $\cX_t$. \label{alg:disc}

\UNTIL{ $S_t = \emptyset$ }
\STATE Let $\hsigma$ be the permutation sorting $(\hmu^t_i)_{i \in \cX_t}$ into ascending order. 
\STATE {\bf Output:} $\heta(x) = \hsigma\left(\argmin_{j \in \cX_t}(   ||j - x||_d)\right)$
\end{algorithmic}
\end{algorithm}

\begin{figure}
\begin{minipage}[c]{0.45\linewidth}
\includegraphics[width=\linewidth]{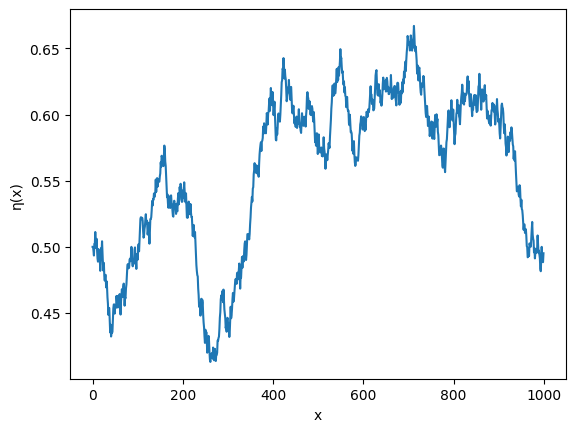}
\end{minipage}
\hfill
\begin{minipage}[c]{0.45\linewidth}
\includegraphics[width=\linewidth]{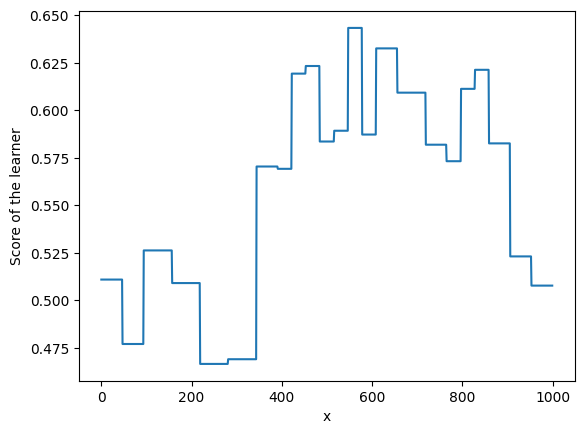}
\end{minipage}%
\caption{(Left) Regression function $\eta$, generated by random walk. (Right) Scoring function outputted by \klcrank.}\label{fig:2}
\end{figure}
As mentioned, for \klcrank, the final level of discretisation will differ across the feature space. The goal is for the local level of discretisation, to ultimately match the local gap $\Delta(x)$. Otherwise, our discretisation error is either too large, or we sample an unnecessary number of points. As an illustrative example, in Figure \ref{fig:2} we see both the true regression function $\eta$ and the scoring function outputted by \klcrank. We see the algorithm has a higher level of discretisation where $\eta$ is more flat, as the gap, $\Delta(x)$ will be smaller on this section.  In addition to our varying level of discretisation, our algorithm \klcrank is further distanced from \rankmessy as follows. Instead of pulling all points in the active set at each time step, it samples only the one with maximum gap. This change allows us to capture the true dependence on KL divergence based confidence intervals, an open question stated in \cite{cheshire2023active}. 
\subsection{Theoretical guarantees for \klcrank}
In the following Theorem we demonstrate that \klcrank is PAC$(\epsilon,\delta)$, as in Equation \eqref{eq:pac} and provide an upper bound on its expected sampling time.

\begin{theorem}\label{thm:packlcrank}
 For $\epsilon, \delta > 0$, with $\beta(t,i,\delta) = c\log(t^2\hDelta_{i,t}^{-d/\beta}/\delta)$ where $c>0$ is an absolute constant, on all problems $\nu \in \cB$, we have that \klcrank is PAC$(\epsilon,\delta)$, and it's expected sampling time is upper bounded by, 
   $$c '\int_{x\in [0,1]^d} H(x) \log\left(c''  H(x)/\delta\right)\;dx\;,$$
 where $c ', c'' >0$ are absolute constants. 

\end{theorem}

Note that the rate exhibited in Theorem \ref{thm:packlcrank} depends upon the dimension $d$, through the sample complexity $H(x)$. 
\subsection{Sketch proof of Theorem \ref{thm:packlcrank}}

The proof of Theorem \ref{thm:packlcrank} can be split into two parts. The first is to show that the algorithm \klcrank is PAC$(\epsilon,\delta)$, see Lemma \ref{lem:pac} in the Appendix, the second is to then upper bound the expected stopping time of \klcrank,
see Lemma \ref{lem:exp} in the Appendix. We will now provide a sketch of proof. All referenced Propositions and Lemmas, can be found in the Appendix, along with their respective proofs.

\paragraph{Proving \klcrank is PAC$(\epsilon,\delta)$}To prove \klcrank is PAC$(\epsilon,\delta)$, we will consider the favourable event, where at each time step, the true values of $\eta$ at the sampled points, are within their respective confidence intervals. We define and bound the probability of such an event in the following Lemma. 
\newtheorem*{lem:goodev}{Lemma \ref{lem:goodev}}
\begin{lem:goodev}\label{lem:goodevm}
We have that the event,
\begin{equation}
\mathcal{E}= \bigcap_{t \in \mathbb{N}}\bigcap_{i \in [\cX_t]} \bigg\{\eta(i)\in [\lcb(t,i),\ucb(t,i)]\bigg\} \;\cap  \bigg\{p \in [\lcb(t,0),\ucb(t,0)]\bigg\}  \;,
\end{equation}
occurs with probability greater than $1-\delta$. 
\end{lem:goodev}

The proof of Lemma \ref{lem:goodevm} can be found in Appendix \ref{app:pac}. While similar in nature to Lemma 3.1 of \cite{cheshire2023active}, our Lemma \ref{lem:goodevm} does not follow immediately from a concentration inequality and a union bound, the reason being our exploration parameter grows with the size of the grid. Instead we apply multiple union bounds over many grids, which increase in size geometrically. 

To suffer less than $\epsilon$ regret, we know intuitively that, for any given $x$, the learner must be able to correctly rank the pair $x,y$ for all $y : |\eta(x) - \eta(y)| \leq \Delta(x)$. Our careful choice of elimination rule ensures that points are only removed from the active set $S_t$, when we are confident that this condition is satisfied. The following Lemma demonstrates this property. 
\newtheorem*{lem:welldist}{Lemma \ref{lem:welldist}}
\begin{lem:welldist}
 Upon execution of \klcrank, on event $\cE$, for all $x \in [0,1]^d$, we have that, $\forall y : |\eta(x) - \eta(y)|\leq \Delta(x),$
$$
\mathrm{sign}(\heta(x) - \heta(y)) = \mathrm{sign}(\eta(x) - \eta(y))\;.
$$
\end{lem:welldist}

The proof of Lemma \ref{lem:welldist} follows from Propositions \ref{prop:St} and \ref{prop:St2}, found in the Appendix. The remainder of proof that \klcrank is PAC$(\epsilon,\delta)$, then follows as in \cite{cheshire2023active}, from analysis of the ROC curve of $\hat \eta$, under the condition of Lemma \ref{lem:welldist}.

\paragraph{Upper bounding the expected sampling time of \klcrank} 
Having demonstrated that \klcrank is PAC$(\epsilon,\delta)$, to complete the proof of Theorem \ref{thm:packlcrank}, it remains to upper bound the expected total number of samples \klcrank draws across its running time. 

In the discrete setting, as studied in \cite{cheshire2023active}, $\eta$ is piece wise constant on some uniform grid of known size $K$ and one can then upper bound the total expected number of samples, by upper bounding the expected number of samples on each individual section of the grid. In our continuous setting this strategy is no longer viable, as our level of discretisation will not be fixed across the feature space. However, we can expect certain subsections of the feature space, on which the gap $\Delta(x)$ does not vary too much, to have similar levels of discretisation. To upper bound the expected number of samples drawn by \klcrank on the feature space $[0,1]^d$, we will first upper bound the expected number of total samples on a given subset of the feature space $W \subset [0,1]^d$. We write $N(W)$ for the total number of samples the learner draws on $W$ and $\Delta_W = \min_{x\in W}(\Delta(x))$. In what follows, $c,c'$ are absolute constants which change line by line. First we define the sequence $(t_i)_{i\in \mathbb{N}}$, where 
\begin{equation}
t_i := \argmin\Bigg\{s : \frac{\log(s^2\Delta_W^{-d/\beta}/\delta)}{s} \leq \max_{x \in W}\kl\left(\eta(x),\eta(x) + 2^{-i}\Delta_{W}/120\right)\Bigg\} \;,
\end{equation}
we suppress the dependence on $W$ in the notation $t_i$. Roughly speaking once points in $W$ have been sampled $t_i$ times, we can expect their confidence intervals to be tight enough, such that they are removed from the active set with probability quickly converging to 1 as $i$ increases. As the actions of the algorithm are indexed by a global time $t$, we define $T_i$ as the global time at which a point in $W$ is first sampled $t_i$ times, that is, $T_i := \min\left(s: \exists j \in \cX_s\cap W, N_{j}(s) \geq t_{i}\right)$. We are now ready to define our "good event" $\xi_{W,i}$. Essentially on this event, at time $T_i$, when a point in $W$ is first sampled $t_i$ times, the empirical means of all active points, $\cX_{T_i}\cap S_t$,  will be within a distance $\Delta_W$ multiplied by a small constant, to their true means, specifically we define, 
$$\xi_{W,i} := \{\forall j \in \cX_{T_i}\cap S_{T_i}, |\hmu_{j}^{T_i} - \eta(j)| \leq \Delta_{W}/120  \}.$$ 
First, we upper bound the probability of the event $\xi_{W,i}$ not occurring, 
\begin{equation}\label{eq:xim}
  \P(\xi_{W,i}^c)  \leq c 2^{-i}t_i^{-3}\;,
\end{equation} 
where $c>0$ is an absolute constant, see Proposition \ref{prop:probxi}. We then go on to show that on the event $\xi_{W,i}$, all points in $W$ are eliminated from the active set by time $T_i$, i.e. $S_{T_i} \cap W = \emptyset$ see Proposition \ref{prop:xi}. While Proposition \ref{prop:xi} uses similar techniques to the proof of Lemma A.6 in \cite{cheshire2023active}, these techniques themselves adapted from the fixed confidence best arm identification literature, our event $\xi_{W,i}$ is fundamentally different to that considered in \cite{cheshire2023active}. The remaining components of the proof, up to Proposition \ref{prop:fin} deal with difficulties specific to the continuous setting and have no counterpart in \cite{cheshire2023active}.

The total number of samples the learner draws on $W$ does not only depend upon the number of times individual points in $W$ are sampled, but also the number of distinct points sampled in $W$. The number of points in $W$ depends upon the discretisation level of the algorithm, $\Delta_{(T_i)}$ in that we can upper bound the number of points in $W$ as follows,
$$|\cX_t \cap W| \leq \Delta_{(T_i)}^{-d/\beta}\lambda( W)\;,$$
however, we would like our upper bound to depend on $\Delta_W$, not $\Delta_{(T_i)}$ as the latter is a random variable. Thus, we demonstrate, that on the event $\xi_{W,i}$ we  $\Delta_{(T_i)}$ is close to $2^{-i}\Delta_{W}$, see Proposition \ref{prop:sizeD}. This result then gives us our upper bound on the number of points in $W$,
$$|\cX_t \cap W| \leq (2^{-i}\Delta_{W})^{-d/\beta}\lambda(\bar W)\;,$$
and from this a bound on the total number of samples drawn from points in $W$. 
\begin{equation}\label{eq:he}
N(W) \leq \lambda( W)(2^{-i}\Delta_W)^{-d/\beta} t_{i}\;,
\end{equation}
Via a combination of Equations \eqref{eq:xim} and \eqref{eq:he}, we have that for all $i>0$
$$\P \left(N(W) \leq \lambda( W)(2^{-i}\Delta_W)^{-d/\beta} t_{i} \right) \leq c 2^{-i}t_i^{-3}\;.$$
Considering the above for all $i \in \mathbb{N}$, we are then able to demonstrate the following upper bound on $\E[N(W)]$,
\begin{equation}\label{eq:10}
\frac{\E[N(W)]}{\lambda(W)}\leq c\Delta^{-d/\beta}_{W}\max_{x\in W}  H(x)\log(c'  H(x)/\delta) \;,
\end{equation}
see the proof Proposition \ref{prop:Wbound} for the exact derivation of Equation \eqref{eq:10}. 

Finally we consider the sequence of sets, indexed by $n,k \in \mathbb{N}$, as, 
$$G_{n,k} =\left\{x : \Delta(x) \in [2^{-n},2^{-n-1}], H(x) \in[2^{-k},2^{-k-1}] \right\}\;.$$ 
We then use \eqref{eq:10}, for the following upper bound,
$$\sum_{n,k=1}^\infty \E[N(G_{n,k})] \leq c\sum_{n,k\in \lambda(G_{n,k})\mathbb{N}}\max_{x\in G_{n,k}}\Delta(x)^{-d/\beta}H(x)\log(c'  H(x)/\delta)\;,$$
where $c>0$ is an absolute constant, see Proposition \ref{prop:fin}. It then remains to lower bound the integral, 
$$ c\sum_{n,k\in \lambda(G_{n,k})\mathbb{N}}\max_{x\in G_{n,k}}\Delta(x)^{-d/\beta}H(x)\log(c'  H(x)/\delta)\leq c'\int_{x\in [0,1]^d} \frac{ \Delta(x)^{-d/\beta}\log(H(x))}{\kl(\eta(x) - \Delta(x), \eta(x) + \Delta(x))}\;dx\;,$$
where $c,c'>0$ are absolute constants. 

\subsection{Lower bound}

In the following theorem we demonstrate a lower bound on the expected sampling time of any PAC$(\epsilon,\delta)$ algorithm, for the 1 dimensional case, the proof can be found in section \ref{app:lb}.
\begin{theorem}\label{thm:lb}
Let $\epsilon \in [0,1/4), 0<\delta < 1-\exp(-1/8), d=1$ and $\nu \in \cB$. For any PAC$(\epsilon, \delta)$ policy $\pi$,
there exists a problem $\bar \nu \in \cB$ such that, for all $x \in [0,1]$, $\bar \Delta(x) \geq \Delta(x)/2$ where $\bar \Delta(x)$ is the gap of point $x$ on problem $\bar \nu$, where the expected stopping time of policy $\pi$ on problem $\bar \nu$ is bounded as follows,

$$c'\int_{x\in[0,1]^d}\bar H(x)\; dx\;,$$

where $c'>0$ is an absolute constant and $\bar H(x)$ is the complexity of point $x$ on problem $\bar \nu$.

\end{theorem}
We see that the above lower bound matches our upper bound on the expected sampling time of \klcrank, see Theorem \ref{thm:packlcrank}, up to logarithmic terms. 
\subsection{Continuous label fixed threshold setting}
In the continuous label fixed threshold setting, the setup is as follows. We no longer restrict the labels $Y$ to take a binary value in $\{0,1\}$, but rather allow for continuous labels, taking a value in $\mathbb{R}$. That is, when sampling a point $x\in [0,1]^d$, the learner observes the realisation of some random variable, with distribution function $F_x$. For a fixed threshold $\rho\in (0,1)$, known to the learner, we can then define the regression function $\eta_\rho(x) = 1-F_x(\rho)= \P(Y\geq \rho | X=x)$. For $\rho\in (0,1)$, the ROC curve of a scoring function $s$, is then the PP-plot $t\in \mathbb{R}\mapsto (1-H_{s}^\rho(t),\; 1-G_s^\rho(t))$, where $H_{s}(t)=\mathbb{P}\left\{ s(x)\leq t \mid Y < \rho  \right\}$ and  $G_{s}(t)=\mathbb{P}\left\{ s(x)\leq t \mid Y \geq \rho\right\}$, for all $t\in \mathbb{R}$. Now the set $\mathcal{S_\rho}^*$ of optimal scoring functions is composed of increasing transforms of the posterior probability $\eta_\rho(x)=\mathbb{P}\{Y\geq \rho \mid X=x  \}$. The setup is then exactly the same as in the binary label setting, the learner must ensure, given threshold $\rho>0$, their estimated scoring function $\heta_\rho$, satisfies, $d_\infty(\eta_\rho,\heta_\rho)\leq \epsilon$, with probability greater than $1-\delta$, while minimising their expected number of samples. 

Our results extend to this setting, see Appendix section \ref{app:cont} where we propose the algorithm\newline \kltcrank, which is similar to \klcrank but makes it's decisions on confidence intervals based on the Dvoretzky–Kiefer–Wolfowitz inequality as opposed to the KL divergence. 

\section{Numerical Experiments}

\begin{figure}
\begin{minipage}[b]{0.49\linewidth}
\includegraphics[width=\linewidth]{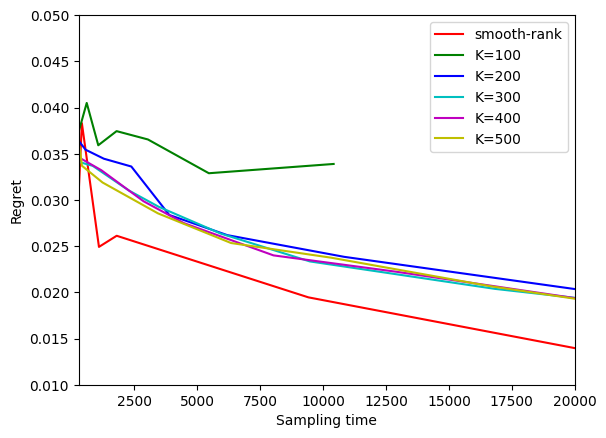}
\end{minipage}
\hfill
\begin{minipage}[b]{0.49\linewidth}
\includegraphics[width=\linewidth]{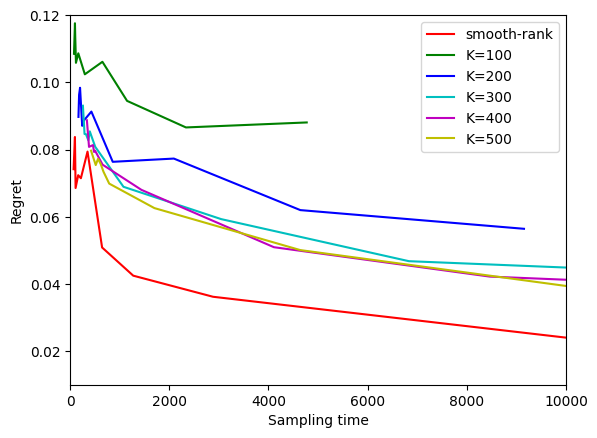}
\end{minipage}%
\vfill
\begin{minipage}[b]{0.49\linewidth}
\includegraphics[width=\linewidth]{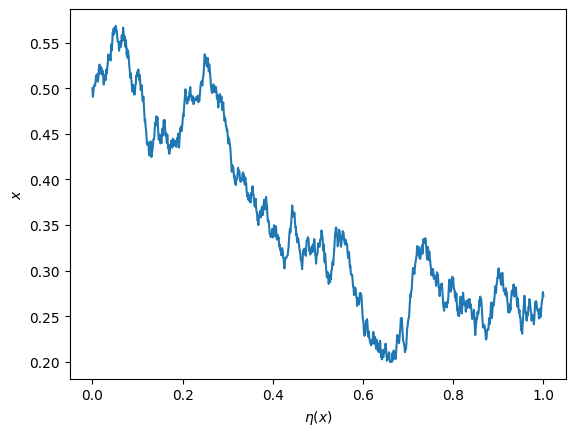}
\end{minipage}
\hfill
\begin{minipage}[b]{0.49\linewidth}
\includegraphics[width=\linewidth]{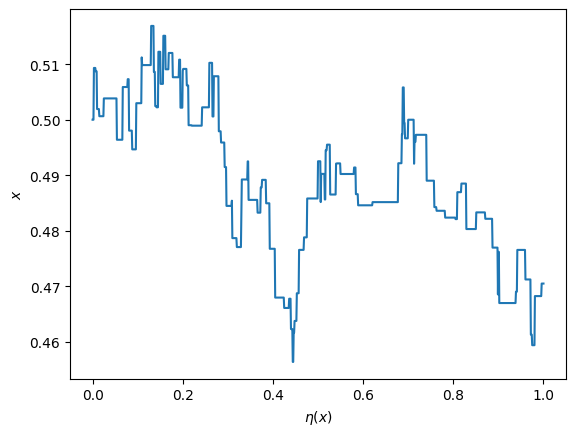}
\end{minipage}%
  \caption{(Top left, top right respectively) Performance of \klcrank compared with \rankmessy, for scenario 1 and scenario 2, for K=(100,200,300,400,500), regret estimated by 50 Monte Carlo realisations of each algorithm. (Bottom left, bottom right respectively) Example simulation of scenario 1 and 2. }\label{fig:3}
\end{figure}

To the best of our knowledge, ranking an $\cX$-armed bandit has not been considered in the literature. We can instead compare with discrete methods, adapted to the $\cX$-armed setting. The most suitable candidate is the \rankmessy algorithm of \cite{cheshire2023active}, as they consider directly the discrete version of our setting and the \rankmessy algorithm is shown to perform well against other potential competitors. The algorithm \rankmessy takes as a parameter $K$, relating to the piece wise constant assumption of \cite{cheshire2023active}, see Equation \eqref{eq:piececont}. As in our setting we make no such assumption, we will run the \rankmessy algorithm across multiple values of $K$. We compare the performance of the algorithms \rankmessy and \klcrank, on 1-Hölder smooth regression functions $\eta$, where each $\eta$ is a step function, where the step values correspond to the path taken by some random walk on the interval $[0,1]$. We consider two scenarios: in scenario 2, $\eta$ is generated by a classical random walk, that is, if $b_t$ is the state of the walk at time $t$, $b_{t+1} = b_t + \phi_t$, where $\phi_t$ is iid, centered truncated gaussian noise. In scenario 1 the random walk used to generate $\eta$, when progressing from time $t$ to $t+1$, will remain constant with probability 0.9, i.e. with probability 0.9, $b_{t+1} = b_t$ and otherwise, $b_{t+1} = b_t + \phi_t$, as in Scenario 2. The result is that, for scenario 1 the gap $\Delta(x)$ will remain constant across large sections of the feature space, as opposed to scenario 2 where it will change more frequently. Thus, one could expect algorithms designed for the discrete case to perform better on scenario 1 than 2. This appears to be the case, Figure \ref{fig:3} shows how the regret, i.e. $d_\infty(\heta,\eta)$, of \klcrank and \rankmessy evolve with increasing sampling time. We see that \klcrank performs favourably against \rankmessy, especially for small sample times, with the difference being more striking in scenario 2. 

\subsection{Experiments on simulated credit risk data}
Carrying out experiments on real world data, in an active setting has inherent difficulties, i.e. one cannot simply train a classifier on already labeled data. As such, we propose to simulate an active setting, via a real world data set. Credit risk evaluation has been a key motivating application for bipartite ranking in the batch setting. We will use the \href{https://www.kaggle.com/code/rishabhrao/home-credit-default-risk-extensive-eda}{Home Credit Default Risk Extensive EDA} data set, which shows the credit default for over 30000 users, with variety of features. We focus on the features credit and annuity. The regression function, i.e. the posterior probability $\eta$ being approximated via kernel regression, with the optimal bandwidth chosen via cross validation. The performance of our smooth-rank and the active-rank of \cite{cheshire2023active} are then to be evaluated for said $\eta$. As with our synthetic experiments, we run the active-rank of \cite{cheshire2023active} for a variety of $K$, this essentially gives the active-rank algorithm the "benefit of the doubt" and assumes one is somehow able to access the optimal or near optimal value of $K$, which in practice would require knowledge of the gaps $\Delta(x)$, see the discussion in Section \ref{sec:lit}. In Figure \ref{fig:4} we see that \klcrank  performs well against active-rank when we simulating credit default rate, given user credit. Increasing the grid size $K$ appears to have little effect on the performance of \rankmessy, suggesting a fixed level of discretisation is unsuitable in this setting. In the case of simulating credit default given user annuity, active-rank again outperforms \rankmessy for in for low sampling times, however, as the sampling times of the algorithms increases, the performance of \klcrank begins to flat line, while that of \rankmessy continues to decrease, eventually outperforming \klcrank. We conjecture that the limiting factor in the performance of \klcrank is down to it's fixed assumption on the smoothness constraint.  The connection between the bandwidth of the kernel regression and the smoothness parameter passed to the algorithm is purely heuristic. It would be of interest to consider practical settings for which the smoothness parameter is known exactly and choosing a reasonable smoothness parameter $\beta$ to pass to \klcrank, remains an interesting question for the general case. 

In the above experiments, we only consider a small subset of the features, this is to keep running time manageable. Ideally one would do feature selection via say PCA, again, we leave this for a more application oriented work. As no attempt was made to minimise the constants appearing in Theorem \ref{thm:packlcrank}, they are likely to be largely overestimated. In the running of our experiments, the constants used in \klcrank, differ from their theoretical counterparts.
\begin{figure}
\begin{minipage}[b]{0.49\linewidth}
\includegraphics[width=\linewidth]{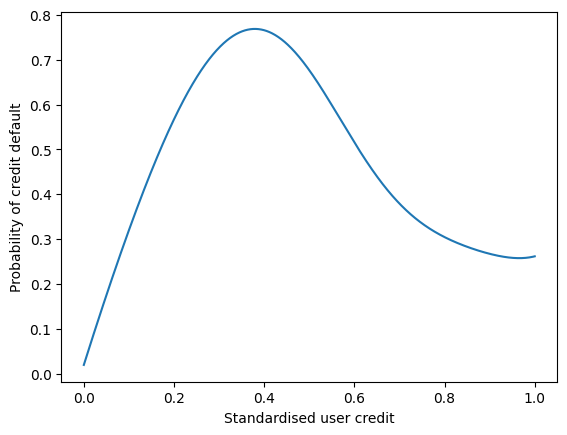}
\end{minipage}
\hfill
\begin{minipage}[b]{0.49\linewidth}

\includegraphics[width=\linewidth]{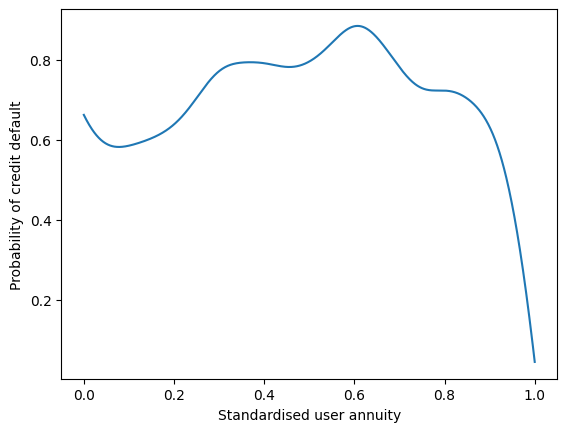}

\end{minipage}
\vfill
\begin{minipage}[b]{0.49\linewidth}
\includegraphics[width=\linewidth]{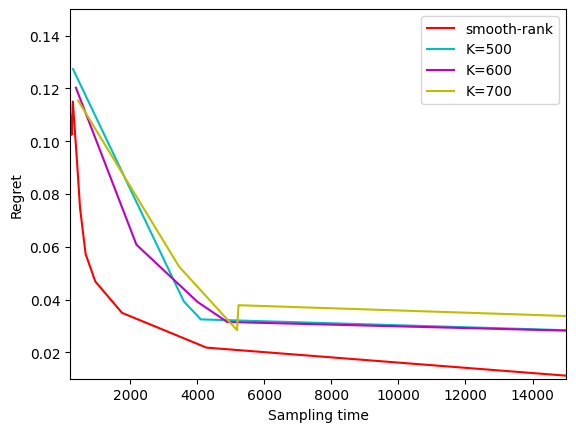}
\end{minipage}
\hfill
\begin{minipage}[b]{0.49\linewidth}
\includegraphics[width=\linewidth]{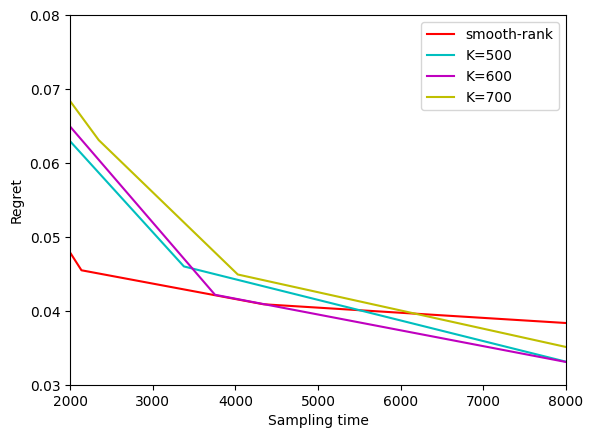}
\end{minipage}%
  \caption{(Top right, top left respectively) 
  Result of KDE for modeling credit default risk given user credit and annuity on EDA data.  (Bottom right, bottom left respectively) Performance of \klcrank compared with \rankmessy, for credit default given user credit and annuity, for K=(500,600,700), regret estimated by 50 Monte Carlo realisations of each algorithm.}\label{fig:4}
\end{figure}

\section{Discussion and Perspectives}\label{sec:conc}

\paragraph{Adaptation to unknown smoothness}
Following the work done in optimisation, e.g. \cite{grill2015black} and \cite{pmlr-v98-bartlett19a}, where the authors consider optimisation of the $\cX$-armed bandit under unknown smoothness, a natural question is what happens when the smoothness parameter $\beta$ is unknown in our bipartite ranking setting. However, in optimisation, adaptation to an unknown smoothness parameter is somewhat easier than in ranking. The reason being, one can take a sub routine, which requires a known smoothness parameter and run many such sub routines in parallel for varying degrees of smoothness. The output of the sub routine with best performance can then be chosen. However, in ranking - as opposed to optimisation, it is not so simple to judge the best ranking rule from a variety of contenders, or more specifically to our setting, judge when a ranking rule guarantees less than epsilon regret. Thus in our opinion, adaptation to smoothness constraints in active ranking, represents a unique and challenging problem, distinct from that faced in optimisation of continuous armed bandits.


\paragraph{Multi partite and continuous label ranking} Our results extend to the case where one observes a continuous label $Y$, and ranks the feature space according to the probability $Y$ is above some fixed threshold $\rho$, known to the learner, see section \ref{app:cont} of the appendix. The case where the threshold $\rho$ is not fixed - instead, for instance, equipped with some prior, remains an open problem. Also of interest is multipartite ranking, in this case the label $Y$ takes a finite number of values, as opposed to being binary as in our setting. An active multipartite ranking would be closely related the ranking experts problem, see \cite{pmlr-v202-saad23b}, but with the modification that, for each query the learner observes the performance of an expert on a single problem, chosen at random, as apposed to the experts performance on all problems simultaneously. 
\subsubsection*{Acknowledgements}
The work of J.~Cheshire is supported by the FMJH, ANR-22-EXES-0013. 
\newpage

\bibliography{bibly.bib}

\appendix

\newpage
\onecolumn

\section{Extension of results to continuous label fixed threshold setting}\label{app:cont}
In the continuous label fixed threshold setting, the setup is as follows. We no longer restrict the labels $Y$ to take a binary value in $\{0,1\}$, but are rather allow for continuous labels, taking a value in $\mathbb{R}$. That is, when sampling a point $x\in [0,1]^d$, the learner observes the realisation of some real valued random variable, with distribution function $F_x$. For a fixed threshold $\rho\in [0,1]$, known to the learner we can then define the regression function $\eta_\rho(x) = 1-F_x(\rho)= \P(Y\geq \rho | X=x)$. The setting then follows as in the bipartite case, indeed bipartite ranking can be seen as a restricted version of the continuous label setting. Our approach in the continuous label setting, will mirror that for bipartite ranking, the main difference being that our algorithm will base its decisions on confidence bounds constructed via the Dvoretzky–Kiefer–Wolfowitz (DKW) inequality, as opposed to the KL divergence. The DKW inequality controls the probability that an empirical distribution function differs from its true counterpart. Specifically, for some iid real valued random variables $X_1,...,X_n$ with distribution function $F$, we can define the empirical distribution, 

$$\hF^n(x) = \frac{1}{n}\sum_{i=1}^n \Ind(X_i \geq x)\;,$$
the DKW inequality then states, for $\epsilon > 0$,

$$\P\left(\sup_{x\in\mathbb{R}}|\hF^n(x) - F(x)| \geq \epsilon\right) \leq c\exp\left(-2n\epsilon^2\right)\;.$$
As the DKW is known to be tight in the worst case, it is reasonable to then update our sample complexity for the continuous label case as follows,

$$H(x) = \Delta(x)^{-d/\beta-2}\;.$$

As one can see, the DKW inequality is not locally dependent, i.e. for a fixed number of samples the width of the generated confidence bound is not dependent on the value of the respective empirical mean, as opposed to KL divergence based confidence intervals which become tighter as one approaches 0 or 1. As a result, our algorithm for the continuous label case is simpler than \klcrank. We no longer sample the point in the active set with largest confidence interval, but rather progress round by round, where at each round all all points in the active set are sampled. That is, $t$ no longer tracks the time step of the algorithm, with the algorithm taking one sample per time step. Instead, $t$ tracks the number of rounds, with the algorithm drawing many samples per round. As a result, for a given round $t$, the widths of the confidence bound, on all points in the active will be constant, denoted $\Delta_{(t)}$, where, 

$$\Delta_{(t)} := \argmin\left\{\Delta \in [0,1] : \Delta \geq  \sqrt{\frac{\log(t^2 \Delta^{-d/\beta}/\delta)}{t}}\right\}\;.$$
With our new confidence interval in mind, the elimination rule then follows as for \klcrank,
$$\cQ_t := \left\{i \in \cX_t \cap S_t: \Delta_{(t)}\leq  \left(\frac{\epsilon \hat p}{\Delta_{(t)}^{d/\beta}|U_{i,t}(6\Delta_{(t)})|}\wedge 1\right)  (1 - \hmu_{i}^t)\right\}\;.$$  
At round $t$, for a point $i\in \cX_t$, let $\hF_i^t(\rho)$ denote the empirical distribution function of the point $i$, based on all samples collected from point $i$ thus far. As in \klcrank, we will maintain an estimate of the overall proportion $\P(Y\geq \rho)$, denoted $\hF^t_{0}(\rho)$. 
\begin{algorithm}[H]
\caption{\kltcrank}
\label{alg:kltcrank}
\begin{algorithmic}
\STATE {\bf Input:}  $\epsilon,\delta,\beta,\rho$
\STATE {\bf Initialise:} $S_0 = [0,1]^d$, $\cX_0 = \cE_d(0.5)$, $t=1$
\REPEAT

\STATE  Update the empirical mean $\hF_0^t(\rho)$ by $x_t$ be a point drawn uniformly from $[0,1]^d$,\\ 


\STATE Sample all points $i \in \cX_t$ once and update empirical means $\hF_i^t(\rho)$ .

\FOR{$i \in S_t \cap \cX_t$}
    \IF{$i \in \cQ_t, \Delta_{(t)} \leq \hF_0^t/4$}
  \STATE $S_t = S_t \setminus \{x: \argmin_{j \in \cX_t}(   \lVert j - x \rVert_d)= i\}$

     \ENDIF
\ENDFOR
\STATE Let $\tilde n = \min\left(n : 2^{-n} \leq \Delta_{(t)}^{1/\beta}\right)$. Add the points $\cE_d(2^{-\tilde n }) \setminus \cX_t$ to $\cX_t$.

\STATE $t=t+1$
\UNTIL{ $S_t = \emptyset$ }
\STATE Let $\hsigma$ be the permutation sorting $(\hF^t_i(\rho))_{i \leq |\cX_t|}$ into ascending order. 
\STATE {\bf Output:} $\heta(x) = \hsigma\left(\argmin_{j \in \cX_t}(   ||j - x||_d)\right)$
\end{algorithmic}
\end{algorithm}
The following Lemma mirrors Lemma \ref{lem:goodev}, adapted to the continuous label setting,
\begin{lemma}\label{lem:goodevT}
We have that the event,
    $$\mathcal{E}= \bigcap_{t \in \mathbb{N}}\bigcap_{i \in [\cX_t]} \{|\hF_i^t(\rho) - F_i(\rho)| \leq \Delta_{(t)}\} \cap \{|\hF_{0}^t(\rho)- \P(Y>\rho) | \leq  \Delta_{(t)}\}  \;,$$
    occurs with probability greater than $1-\delta$. 
\end{lemma}
\begin{proof}
Let $t>0$ and $i \in \cX_t$, via the DKW inequality we have that, 

$$\P\left(|\hF_i^t(\rho) - F_i(\rho)| \leq \Delta_{(t)}\right) \leq \frac{\delta \Delta_{(t)}^{d/\beta}}{t^2}$$
The proof of Lemma \ref{lem:goodevT} then follows as in the proof of Lemma \ref{lem:goodev}.\end{proof} With Lemma \ref{lem:goodevT} in mind, as our elimination rule remains unchanged, the proof that \kltcrank is PAC$(\epsilon,\delta)$, demonstrated in the following Lemma, follows as in the proof of \ref{lem:pac}.

\begin{lemma}\label{lem:tpac}
For $\epsilon, \delta > 0$, on all problems $\nu \in \cB$, we have that \kltcrank is PAC$(\epsilon,\delta)$.
\end{lemma}
In the following Lemma we demonstrate an upper bound on the expected sampling time of \kltcrank. 
\begin{lemma}\label{lem:samp}
For $\epsilon, \delta > 0,\rho\in (0,1)$, on all problems $\nu \in \cB$, we have that the expected sampling time of \kltcrank is upper bounded by, 
   $$c '\int_{x\in [0,1]^d} H(x) \log\left(c''  H(x)/\delta\right)\;dx\;,$$
 where $c '$, $c''$ are absolute constants. 

\end{lemma}
\begin{proof}
    
To upper bound the expected sampling time of \kltcrank it is straightforward to adapt the proof of \ref{lem:exp}, with some slight modifications. First we redefine the $t_i$s as follows, 

$$t_i = \argmin\left\{t : t \geq \frac{\log\left(2^i t^2\Delta_{(t)}^{-1/\beta}/\delta\right)}{\Delta_W^2}\right\}$$
and then consider the slightly reformulated "good event",
$$
\xi_{W,t} := \left\{\forall j \in \cX_{t}\cap S_{t}, |\hF_{j}^{t}(\rho) - F_j(\rho)| \leq \Delta_{W}/120  \right\}\cap \{|\P(Y\geq \rho) - \hat F_0^t(\rho)| \leq p/2\}\;.
$$
For $t \geq t_i$, via the DKW inequality we have that for $i \in \cX_t$,
$$\P\left(|\hF_{i}^{t}(\rho) - F_i(\rho)| \leq \Delta_{W}/120 \right) \leq \exp\left(-ct\Delta_{W}^2\right)\leq \delta t^{-2} 2^{-i}\Delta_{(t)}^{-1/\beta}\;.$$
and thus via a union bound, $\P(\xi_{W,t}^c) \leq \delta t^{-2} 2^{-i}$. This matches the result of Proposition \ref{prop:probxi}. What remains is to upper bound $t_0$ similarly and the proof then follows as in the proof of Lemma \ref{lem:exp}. 
\end{proof}
\section{Proof that \klcrank is PAC$(\epsilon,\delta)$}\label{app:pac}
In this section we will prove the following Lemma which demonstrates the PAC$(\epsilon,\delta)$ guarantee for our algorithm \klcrank. 
\begin{lemma}\label{lem:pac}
For $\epsilon, \delta > 0$, with $\beta(t,i,\delta) = c\log(t^2\hDelta_{i,t}^{-d/\beta}/\delta)$ where $c$ is an absolute constant, on all problems $\nu \in \cB$, we have that \klcrank is PAC$(\epsilon,\delta)$. 
\end{lemma}
The proof of Lemma \ref{lem:pac} will follow from several sub lemmas and propositions. Firstly we bound the probability of a good event, on which all empirical means generated across the run time of the algorithm, are close to their true means. 
\begin{lemma}\label{lem:goodev}
We have that the event,
    $$\mathcal{E}= \bigcap_{t \in \mathbb{N}}\bigcap_{i \in [\cX_t]} \{\mu_i \in [\lcb(t,i),\ucb(t,i)]\} \cap \{p \in [\lcb(t,0),\ucb(t,0)]\;,$$
    occurs with probability greater than $1-\delta$. 
\end{lemma}
\begin{proof}
Let $T_{(i)}$ be the set of time points for which the set of points $\cX_t$ is of size between $2^i$ and $2^{i+1}$, that is,
$$T_{(i)} = \{t : 2^i\leq |\cX_t| \leq 2^{i+1}  \}\;,$$

let $\cX_{(i)}$ be the set of points, that are sampled on points of time in $T_{(i)}$, that is, 

$$\cX_{(i)} := \{ j : \exists t \in T_{(i)} : a_t = j\}\;.$$ 
We can rewrite event $\cE$ as, 
$$\bigcap_{i\in \mathbb{N}}\bigcap_{j\in \cX_{(i)} }\bigcap_{t\in\{ \mathbb{N}:\exists s \in T_{(i)}: N_j(s)=t\}} \{\mu_k \in [\lcb(t,i),\ucb(t,i)]\} \;.$$

The sets $\cX_{(i)}$ are of course themselves random variables, however as everything that follows holds for any possible realisation of the $\cX_{(i)}$s, we can treat them as deterministic. For all $i\in \cX_{(i)}, t \in T_{(i)}$, we have that,  $\hDelta_{i,t}^{d/\beta} \leq 2^i$, and thus, via Chernoff and the choice of exploration parameter, $\beta(t,i,\delta) = c\log(t^2\hDelta_{i,t}^{-d/\beta}/\delta)$, for $i\in \mathbb{N}$, $j\in \cX_{(i)}$, $t\in \{\mathbb{N}:\exists s \in T_{(i)}: N_j(s)=t\}$ we have that, $$\P(\mu_i \notin [\lcb(t,i),\ucb(t,i)]) \leq 2^{-ci}t^{-2c}\delta\;.$$
then, via a union bound
\begin{align}
\P(\cE^c) &\leq \sum_{i\in N}\sum_{j\in\cX_{(i)}}\sum_{t\in\{ \mathbb{N}:\exists s \in T_{(i)}: N_j(s)=t\}} 2^{-ci}t^{-2c}\delta\\
&\leq  \frac{\delta\pi^2}{6}\sum_{i\in N}\sum_{j\in\cX_{(i)}} 2^{-ci}\\
&\leq  \frac{\delta\pi^2}{6}\sum_{i\in N} 2^{i+1}2^{-ci}\\
&\leq  \frac{2\delta\pi^2}{6}\sum_{i\in N} 2^{-i(c-1)}\;,   
\end{align}
as $|\cX_{(i)}| \leq 2^{i+1}$, it remains to choose $c = 2 + \log_2(2\pi^2/6)$ and the result follows.

\end{proof}
Our algorithm \klcrank also maintains an estimate $\hp_t$ of the true proportion $p$ across its run time. The algorithm only eliminates points, when the confidence intervals around $\hp_t$ for $p$ are sufficiently tight, see line \ref{alg:p}. In the following proposition, we demonstrate that whenever the algorithm eliminates an arm at some time $t$, our estimate $\hp_t$ is sufficiently close to the true $p$.  
\begin{proposition}\label{prop:phat}
 On event $\cE$ we have that, for all times $t$ such that a point is removed from the active set, $2p /3\leq \hat p_t \leq 4p/3$.
\end{proposition}
\begin{proof}
If a section is removed from the active set at time $t$, we have that $\Delta_{(t)} \leq \hat p /4$. Now, on event $\cE$ we have $|p - \hat p| \leq \Delta_{( t)}$ which, in combination with the fact $\Delta_{(t)} \leq \hat p /4$, implies $\Delta_{(t)} \leq (\hp + \Delta_{(t)})/3$ and thus $\Delta_{(t)} \leq p/3$.

\end{proof}
In the following two Propositions, we show that on our favorable event $\cE$, the \klcrank only removes sections of the active set $S_t$, when our ranking on them is corrrect. 
\begin{proposition}\label{prop:St}
On event $\cE$, at time $t$, let $i \in \cX_t$ be such that $i \in \cQ_t$ and $\{x : |\eta(i) - \eta(x)| \leq \Delta(i)\} \subset S_t$. For all $x : |\eta(x) - \hmu_i^t| \leq 4\Delta_{(t)}$, we have that, $\forall y \in S_t : |\eta(y) - \eta(x)| \geq \Delta(x),$
$$ \mathrm{sign}(\heta(x) - \heta(y)) = \mathrm{sign}(\eta(x) - \eta(y))\;.$$
\end{proposition}
\begin{proof}

Let us fix an  $x \in S_t$ such that $x : |\eta(x) - \hmu_i^t| \leq 4\Delta_{(t)}$. The first step of the proof will be to show that $\Delta_{(t)} \leq \Delta(x)$.


Now as as $i \in \cQ_t$ at time $t$, we have that,

\begin{align}
    6\Delta_{(t)} &\leq \frac{\epsilon \hp_t (1 - \hmu_{i}^t)}{ \Delta_{(t)}^{d/\beta}|U_{i,t}(6\Delta_{(t)})| }\nonumber \\
    &\leq  \frac{4\epsilon p (1 - \hmu_{i}^t)}{3\Delta_{(t)}^{d/\beta}|U_{i,t}(6\Delta_{(t)})|}\nonumber  \\ 
    &\leq  \frac{4\epsilon p (1 - \eta(x)) + 5\epsilon\Delta_{(t)}}{3\Delta_{(t)}^{d/\beta}|U_{i,t}(6\Delta_{(t)})|}\;.\label{eq:Uboundx}
\end{align}
where the second line comes from Proposition \ref{prop:phat}. We then have that 
\begin{align}
    4\Delta_{(t)} &\leq  \frac{4\epsilon p (1 - \eta(x)) }{3\Delta_{(t)}^{d/\beta}|U_{i,t}(6\Delta_{(t)})|}\\
    &\leq  \frac{4\epsilon p (1 - \eta(x)) }{3\Delta_{(t)}^{d/\beta}|U_{x,t}(4\Delta_{(t)})|}\;.
\end{align}
where the second line comes from the fact that $U_{x,t}(4\Delta_{(t)}) \subset U_{i,t}(6\Delta_{(t)})$. We now note that, via line \ref{alg:disc} of \klcrank, $\Delta_{(t)}^{d/\beta}|U_{x,t}(4\Delta_{(t)}))| \geq \lambda(\{z : |\eta(x) - \eta(z)| \leq 4\Delta_{(t)}\})$ and so we have
$$ \Delta_{(t)} \leq  \frac{4\epsilon p (1 - \eta(x)) }{3\lambda(\{z : |\eta(x) - \eta(z)| \leq 4\Delta_{(t)}\})}$$
and thus $4\Delta_{(t)} \leq \Delta(x)$. Now let us fix a $y \in S_t$ and $j_y \in \cX_t \cap S_t$, such that, $|\eta(y) - \eta(x)| \geq \Delta(x)$ and let $j_y = \argmin_{j\in \cX_t\cap S_t}\left(\lVert y-j \rVert_d\right)$ similarly let $j_x = \argmin_{j\in \cX_t\cap S_t}\left(\lVert x-j \rVert_d\right)$. Showing that 
$$ \mathrm{sign}(\heta(x) - \heta(y)) = \mathrm{sign}(\eta(x) - \eta(y))\;,$$
amounts to demonstrating that
$$ \mathrm{sign}(\heta(j_x) - \heta(j_y)) = \mathrm{sign}(\eta(x) - \eta(y))\;.$$
We see that via line \ref{alg:disc} of \klcrank, $|\eta(j_x) - \eta(x)| \vee |\eta(j_y) - \eta(y)| \leq \Delta_{(t)}$ and then as we have shown $\Delta(x) \geq 6\Delta_{(t)}$, this then implies via the triangle inequality, $|\eta(j_x) - \eta(j_y)| \geq 5\Delta_{(t)}$ and then on event $\cE$,
$$\sign(\hmu_{j_y}^t - \hmu_{j_x}^t) = \sign(\eta(j_x) - \eta(j_y))\;,$$ 
which provides the result.

\end{proof}

\begin{proposition}\label{prop:St2}
On event $\cE$, at time $t$, let $ x \in S_t$, we have that, $\forall y \in [0,1]^d \setminus S_t: |\eta(x)- \eta(y)| \geq \Delta(x),$
$$\mathrm{sign}(\heta(x) - \heta(y)) = \mathrm{sign}(\eta(x) - \eta(y))\;.$$
\end{proposition}
\begin{proof}
The proof will follow by induction. Assume at time $t$ that $\forall x \in S_{t-1}$ we have that, $\forall y \in [S_t] \setminus S_{t-1} : |\eta(x) - \eta(y)| \geq \Delta(x)$, 
$$\mathrm{sign}(\heta(x) - \heta(y)) = \mathrm{sign}(\eta(x) - \eta(y))\;.$$
Now let $x \in S_t$, and $y\in [0,1]^d \setminus S_t: |\eta(x)- \eta(y)| \geq \Delta(x)$, we must show that, 
$$\mathrm{sign}(\heta(\tilde x) - \heta(y)) = \mathrm{sign}(\eta(x) - \eta(y))\;,$$
which amounts to demonstrating that $\sign(\eta(j_x) - \eta(j_y)) = \sign(\hmu_{j_y}^t - \hmu_{j_x}^t)$ where $j_x =\argmin_{j\in \cX_{t-1}\cap S_{t-1}}\left(\lVert x-j \rVert_d\right)$ and $j_y =\argmin_{j\in \cX_{t-1}\cap S_{t-1}}\left(\lVert y-j \rVert_d\right)$.

If $y \notin S_{t-1}$ we are done, via the inductive assumption, thus assume $y \in S_{t-1}$. If $x \in U_{j_y,t-1}(4\Delta_{(t)})$ the proof then follows via Proposition \ref{prop:St}, thus assume $x \notin U_{j_y,t-1}(4\Delta_{(t)})$. In this case, as on event $\cE$, 
$$|\eta(j_y) - \hmu_{j_y}^t|\leq \hDelta_{j_y,t} \leq \Delta_{(t)}\;,$$
via the triangle inequality,
\begin{equation}\label{eq:tri1}
    |\eta(j_y) - \eta(x)| \geq 2\Delta_{(t)}
\end{equation}
we also have via line \ref{alg:disc} of \klcrank, that 
\begin{equation}\label{eq:tri2}
|\eta(y) - \eta(j_y)| \vee |\eta(x) - \eta(j_x)| \leq \Delta_{(t)}\;,
\end{equation}
and thus by combination of Equations \eqref{eq:tri1} and \eqref{eq:tri2},
$$\sign(\eta(x)- \eta(y)) = \sign(\eta(j_x) - \eta(j_y))\;,$$ 
and then on event $\cE$ that, $\sign(\eta(j_x) - \eta(j_y)) = \sign(\hmu_{j_y}^t - \hmu_{j_x}^t)$ and the proof follows. .

\end{proof}
The following Lemma now follows as an immediate corollary of Proposition \ref{prop:St2}. 
\begin{lemma}\label{lem:welldist}
 Upon execution of \klcrank, on event $\cE$, for all $x \in [0,1]$, we have that, $\forall y : |\eta(x) - \eta(y)|\leq \Delta(x),$
$$\mathrm{sign}(\heta(x) - \heta(y)) = \mathrm{sign}(\eta(x) - \eta(y))\;.$$
\end{lemma}
With the above Lemma in hand, it remains to show that our choice of $\Delta(x)$ is appropriate, i.e. we must demonstrate that with the guarantee of Lemma \ref{lem:welldist}, it is implied that $d_\infty(\heta,\eta)\leq \epsilon$. This result will follow  as in the proof of Lemma A.3, \cite{cheshire2023active}. For completeness we include it here, slightly adapted to or setting.
\begin{proof}[Proof of Lemma \ref{lem:pac}]

We reuse the notation of \cite{cheshire2023active}, where for a set $C \subset [0,1]^d$, define, 
$$\kappa(C) := \frac{1}{\lambda(C)}\int_C \eta(x) \;,dx \;.$$
 Let $\alpha \in [0,1]$, define the subset $Z_\alpha \subset [0,1]^d$ such that, $\P(X \in Z_\alpha | Y = -1) = \alpha$ , that is,
$$\frac{\lambda(Z_\alpha)(1 - \kappa(Z_\alpha))}{1-p} =  \alpha\;,$$ 
and such that, for some $x_\alpha \in [0,1]$,
\begin{equation}\label{eq:i}
    \forall y : \eta(y) > \eta(x_\alpha), y \in Z_\alpha\;,\qquad \forall j :  \eta(y) < \eta(x_\alpha), y \notin Z_\alpha\;.
\end{equation}
We then have $\roc^*(\alpha) = \P(X \in Z_\alpha | Y =
 +1) = \frac{\lambda(Z_\alpha)(\kappa(Z_\alpha))}{p}$. The choice of $Z_\alpha$ is not necessarily unique, and as $\eta$ me be constant across sections of the feature space, likewise $x_\alpha$ is also not necessarily unique, in this case we take arbitrary $Z_\alpha$, $x_\alpha$. Now define the subset $\hat Z_\alpha \in [0,1]$ such that, 
$$\frac{\lambda(\hat Z_\alpha)(1 - \kappa(\hat Z_\alpha))}{1-p} =  \alpha\;,$$
and,
$$\forall x \in \hat Z_\alpha, y \notin \hat Z_\alpha, \heta(x) \geq  \heta(y)\;,$$
so $\roc(\heta,\alpha) = \frac{\lambda(\hat Z_\alpha)(\kappa(\hat Z_\alpha))}{p}$. Again $\hat Z_\alpha$ is not necessarily unique, in which case we choose arbitrarily. 
Via Lemma \ref{lem:welldist}, we have that, $\forall y \in [0,1] : |\eta(y)- \eta(x_\alpha)| \leq \Delta(x_\alpha)$,
\begin{equation}\label{eq:allk}
\mathrm{sign}(\heta(y) - \heta(x_\alpha)) = \mathrm{sign}(\eta(y)- \eta(x_\alpha))\;.
\end{equation}

Let 
$$Z_\alpha ' = \{y \in Z_\alpha : |\eta(y) - \eta(x_\alpha)| \leq \Delta(x)\}\;,\qquad \hat Z_\alpha ' = \{y \in \hat Z_\alpha : |\eta(y) - \eta(x_\alpha)|  \leq \Delta(x)\}\;.$$ 
 Via Equation \ref{eq:allk}, we have that, 

 $$\roc(\alpha, \eta) - \roc(\alpha, s_{\hat \cP})  = \frac{\lambda(Z_\alpha ')\kappa(Z_\alpha ')}{p} - \frac{\lambda(\hat Z_\alpha ')\kappa(\hat Z_\alpha ')}{p}\;.$$
Before finalising the proof we must lower bound $\lambda(\hat Z_\alpha ')$ and $\kappa(\hat Z_\alpha ')$. We first lower bound $\kappa(\hat Z_\alpha ')$.

\begin{equation}\label{eq:tec6}
    \kappa(\hat Z_\alpha ')\geq \eta(x_\alpha) - \Delta_{i_\alpha}\;, \qquad \kappa(Z_\alpha '), \leq \eta(x_\alpha) + \Delta_{i_\alpha}\;.
\end{equation}

We will now lower bound $\lambda(\hat Z_\alpha ')$
\begin{equation}\label{eq:tec7}
   \frac{\lambda(\hat Z_\alpha ')}{\lambda(Z_\alpha ')} = \frac{1-\kappa(Z_\alpha ')}{1-\kappa( \hat Z_\alpha ')}\leq \frac{1 - \eta(x_\alpha) + \Delta_{i_\alpha}}{1 - \eta(x_\alpha) - \Delta_{i_\alpha}}\;. 
\end{equation}

Via combinations of Equations \eqref{eq:tec6} and \eqref{eq:tec7}, we have, 
\begin{align}
\frac{\lambda(Z_\alpha ')\kappa(Z_\alpha ')}{p} - \frac{\lambda(\hat Z_\alpha ')\kappa(\hat Z_\alpha ')}{p}&\leq \frac{1}{p}\left(\lambda(Z_\alpha ')(\eta(x_\alpha)  + \Delta(x_\alpha)) - \lambda(\hat Z_\alpha ')(\eta(x_\alpha) - \Delta(x_\alpha) )\right)\\
&\leq \frac{1}{p}\left(\lambda(Z_\alpha ')(\eta(x_\alpha)  + \Delta(x_\alpha) ) - \lambda(Z_\alpha ')\frac{(\eta(x_\alpha)  - \Delta(x_\alpha) )(1-\eta(x_\alpha)  + \Delta(x_\alpha) ))}{1 -\eta(x_\alpha) - \Delta(x_\alpha)}\right)\\
&\leq \frac{2\lambda(Z_\alpha ')\Delta(x_\alpha) }{p(1-\eta(x_\alpha)  - \Delta(x_\alpha) )} \leq \frac{2\lambda(Z_\alpha ')\Delta(x_\alpha) }{p(1-\eta(x_\alpha) )} 
\end{align}
It remains to remark that, $\Delta(x_\alpha) \leq \frac{\epsilon p}{\lambda(Z_\alpha ')}$, by definition, and thus, 
$$\mathrm{ROC}\left(\alpha,\eta\right) - \mathrm{ROC}\left(\alpha,s_{\hat \eta}\right)\leq \epsilon\;.$$
As we chose $\alpha$ w.l.o.g the proof then follows.

\end{proof}
\section{Proof of expected sampling time for \klcrank}\label{app:exp}

In this section we will upper bound the expected sampling time of \klcrank, as stated in the following Lemma. 
\begin{lemma}\label{lem:exp}
For $\epsilon, \delta > 0$, with $\beta(t,i,\delta) = c\log(t^2\hDelta_{i,t}^{-d/\beta}/\delta)$ where $c$ is an absolute constant, on all problems $\nu \in \cB$,  the expected sampling time of \klcrank is upper bounded by, 
$$c '\int_{x\in [0,1]^d} H(x) \log\left(c''  H(x)/\delta\right)\;dx\;,$$
 where $c '$, $c''$ are absolute constants. 
\end{lemma}
The proof of the above Lemma \ref{lem:exp} will follow from several sub lemmas and propositions. For the entirety of this section, $c,c',c''$ are absolute constants which vary line by line. Also for clarity, we restrict to the 1 dimensional case, however, \textit{all results generalise immediately to higher dimensions}. To upper bound the expected number of samples the learner draws on the entire feature space, we will first demonstrate an upper bound on a subset of the feature space. Consider a subset $W\subset [0,1]$. Let $\Delta_{W} := \min_W(\Delta(x))$ be the minimum gap across the subset $W$. For our given subset $W$, we will now define a deterministic sequence of integers $t_0,t_1,...$, here for clarity, we suppress dependency on $W$ in the notation. For $i\in \mathbb{N}$, define,

$$t_i := \argmin\left\{s : \frac{\log(s^2\Delta_W^{-1/\beta}/\delta)}{s} \leq \max_{x \in W}\kl\left(\eta(x),\eta(x) + 2^{-i}\Delta_{W}/120\right)\right\}\;.$$

We will show that, once points in $W$ have been sampled $t_0$ times, they will be eliminated from the active set with constant probability, this probability will converge to 1, as points are sampled $t_1,t_2...$ etc. times. As our notations and the actions of the algorithm are indexed by a global time $t$, we define $T_i$ as the global time at which a point in $W$ is first sampled $t_i$ times, that is for $i\in \mathbb{N}$, 

$$T_i := \min\left(s: \exists j \in \cX_s, N_{j}(s) \geq t_{i}\right)\:.$$
We are now ready to define our "good event" $\xi_{W,i}$. Essentially on this event, at time $T_i$, when a point in $W$ is first sampled $t_i$ times, the empirical means of all active points, $\cX_{T_i}\cap S_t$,  will be within a distance $\Delta_W$ multiplied by a small constant, to their true means, specifically, 
$$
\xi_{W,i} := \{\forall j \in \cX_{T_i}\cap S_{T_i}, |\hmu_{j}^{T_i} - \eta(j)| \leq \Delta_{W}/120  \} \cap \{|\hp_{T_i} - p| \leq p/2\}\;.
$$
We will go on to show that on the good event $\xi_{W,i}$, all points within $W$ will be eliminated from the active set, see Proposition \ref{prop:xi}. Before doing so, we first upper bound the probability of $\xi_{W,i}^c$ in the following proposition. 

\begin{proposition}\label{prop:probxi}
For $W\subset [0,1]$, $i\in \mathbb{N}$, we have that,

$$\P(\xi_{W,i}^c) \leq c2^{-i}t^{-3}_i\;.$$
where $c>0$ is an absolute constant. 
\end{proposition}
\begin{proof}
We fix $i\in \mathbb{N}$ and for clarity of notation, let us set $j= a_{T_i}$, that is the point sampled by the algorithm at time $T_i$, we have that $N_j(T_i) = t_{i}$. Now,

\begin{equation}\label{eq:kl} 
 \P\left(\ucb(t_{i},j) \geq \eta(j) + \Delta_{W}/120\right) \leq \P\left(\kl(\hmu_j^{T_{i}},\eta(j) + \Delta_{W}/120)
 \leq \kl(\eta(j),\eta(j) + \Delta_{W}/120)\right)\;. 
\end{equation}
Let,
$$r(\gamma) = \{z\in (\eta(j), \eta(j) +\Delta_{W}/120) : \kl(z,\eta(j) + \Delta_W/120) = \kl(\eta(j),\eta(j) + \Delta_{W}/120)/\gamma\}\;.$$
Consider the function $\phi(z)= \kl(\eta(j) + z,\eta(j) + \Delta_{W}/120)$, on the interval $[0,\Delta(x) /120]$. We have that $\phi$ is convex and $\phi(\Delta(x)/120) = 0$. Thus for all $z \in \left[0,\frac{\Delta_{W}}{120}\right]$,
\begin{equation}\label{eq:convex}
   \phi(z) \leq (1-z)\frac{120\kl(\eta(j),\eta(j) +\Delta_{W}/120)}{\Delta_{W}}\;. 
\end{equation}
Now consider Equation \eqref{eq:convex} evaluated at $z = r(\gamma) - \eta(j)$, giving,
\begin{equation}\label{eq:eval}
\frac{\kl(r(\gamma),\eta(j) + \Delta_{W}/120)}{\kl(\eta(j),\eta(j) +\Delta_{W}/120)} \leq \frac{120}{\Delta_W}(1+ \eta(j) - r(\gamma))\;.    
\end{equation}
By definition of $r(\gamma)$, we have,
\begin{equation}\label{eq:eval2}
\kl(r(\gamma),\eta(j) + \Delta_{W}/120) = \kl(\eta(j),\eta(j) +\Delta_{W}/120)/\gamma 
\end{equation}
A combination of Equations \eqref{eq:eval} and \eqref{eq:eval2} then leads to,
$$r(\gamma) \geq \eta(j) +\Delta_{W}\left(\frac{1}{120} - \frac{2}{\gamma}\right)\geq  \eta(j)+ \frac{\Delta_{W}}{240}\;,$$ 
for $\gamma > 480$. We now have,
\begin{align}
\nonumber \P\Big(\kl(\hmu_j^{T_{i}},\eta(j) + \Delta_{W}/120) &\leq \frac{\kl(\eta(j),\eta(j) +\Delta_{W}/120)}{\gamma} \Big) \\ \nonumber
&= \P\left(\kl(\hmu_j^{T_{i}},\eta(j) +\Delta_{W}/120)\leq \kl(r(\gamma), \eta(j) + \Delta_{W}/120)\right)\\ \nonumber
&=\P(\hmu_j^{T_{i}} \geq r(\gamma))\\ \nonumber
&\leq c\exp\left(-t_i\kl\left(\eta(j),r(\gamma)\right)\right)\\ \nonumber
&\leq c\exp\left(-\frac{\gamma\kl\left(\eta(j),r(\gamma)\right)}{\kl(\eta(j),\eta(j) + \Delta_{W}/120)}\log(t_i^2\Delta_W^{-1}/\delta)\right)\\ \nonumber
&\leq c\exp\left(-\frac{\gamma\kl\left(\eta(j), \eta(j) + \frac{\Delta_{W}}{240}\right)}{\kl(\eta(j),\eta(j)+ 2^{-i}\Delta_{W}/120)}\log(t_i^2\Delta_W^{-1}/\delta)\right)\\ \nonumber
&\leq c\exp\left(-4^{i}\log(t_i^2\Delta_W^{-1/\beta}/\delta)c_\gamma\right)\\
&\leq c 4^{-i}t_i^{-3}\Delta_{W}^{1/\beta} \label{eq:al}
\end{align}
where $c_\gamma$ is a constant depending only on
 $\gamma$. 
 Thus, via combination of Equations \eqref{eq:kl} and \eqref{eq:al}, 

$$\P\left(\ucb(t_{i} ,j) \geq \eta(j) + \Delta_{W}/120\right) \leq c 4^{-i}t_i^{-3}\Delta_{W}^{-1/\beta}\;,  $$
via similar reasoning we have also that, 
$$\P\left(\lcb(t_{j,i},j) \leq \mu_j - \Delta_{W}/120\right) \leq c 4^{-i}t_i^{-3}\Delta_{W}^{-1/\beta}\;. $$

Now fix some $k\in \cX_{T_i} \cap S_{T_i}$. As point $j$ is sampled at time $T_i$, we have that, $\hDelta_{j,T_i} = \max_{i\in \cX_{T_i} \cap S_{T_i}}\left(\hDelta_{i,T_i}\right)$ and thus $\hDelta_{k,T_i} \leq \hDelta_{j,T_i}$, leading to,
$$\P\left(\lcb(T_i,k) \leq \eta(k) - \Delta_{W}/120\right) \leq \P\left(\lcb(T_i ,j) \leq \eta(j)-\Delta_{W}/120\right)\;,$$
and 
$$\P\left(\ucb(T_i,k) \geq \eta(k) + \Delta_{W}/120\right) \leq \P\left(\ucb(T_i,j) \geq \eta(j) +\Delta_{W}/120\right)\;.$$
And thus,
\begin{equation}\label{eq:new3}
\P\left(  |\hmu_{k}^{T_i} - \eta(k)| \leq \Delta_{W}/120 \right) \leq c 4^{-i}t_i^{-3}\Delta_{W}^{-1/\beta}\;.    
\end{equation}
and similarly as $\hDelta_{0,T_i} \leq \hDelta_{j,T_i}$, thus $|\hp_{T_i} - p| \leq p/2$.

It remains to remark that via the action of the algorithm $|\{\cX_{T_i} \cap S_{T_i}\}| \leq \Delta_{(T_i)}^{-1/\beta} \leq 2^{-i}\Delta_W^{-1/\beta}$,
and thus via a union bound and Equation \eqref{eq:new3}, conditional on $\hDelta_{j,T_i}^{-1/\beta} > 2^{-i}\Delta_W^{-1/\beta}$,

\begin{equation}\label{eq:xi}
  \P(\xi_{W,i}^c)  \leq c 2^{-i}t_i^{-3}
\end{equation} 

 



\end{proof}
To show that all points in $W$ will be eliminated on our good event, we must demonstrate that on event $\xi_{W,i}$, the discretisation level - that is maximum width of the confidence interval, across all active points, $\Delta_{(T_i)}$, is not \textit{too large}. Specifically we need to ensure that $\Delta_{(T_i)}$ is smaller than $\Delta_{W}$ multiplied by some constant. When we later upper bound the expected number of samples the algorithm draws from $W$, it will also be necessary to show the level of discretisation is not \textit{too small}. Thus in the following proposition we both upper and lower bound $\Delta_{(T_i)}$ on the event $\xi_{W,i}$.

\begin{proposition}\label{prop:sizeD}
For $i \in \mathbb{N}$, on the event $\xi_{W,i}$ we have that,
$$ c'2^{-i}\Delta_W\leq \Delta_{(T_i)} \leq c2^{-i}\Delta_W$$

\end{proposition}
\begin{proof}

Let $j= a_{T(i)}$. We have the following,
\begin{align*}
\nonumber
\kl(\hmu_j^{T_i},\hmu_j^{T_i} + c\Delta_W) &\leq \kl(\eta(j) + c\Delta_W,\eta(j) + 2c\Delta_W)\\
&\leq 2\kl(\eta(j),\eta(j) + c\Delta_W)
\;.    
\end{align*}

Now we consider the definition of $t_i$, and apply the above inequality,
\begin{align}
t_i&\geq \log(t_i^2\Delta_W^{-1}/\delta)\max_{j \in W}\kl\left(\eta(j),\eta(j)+ 2^{-i}\Delta_{W}/120\right)\nonumber \\
&\geq \log(t_i^2\Delta_W^{-1}/\delta)\max_{j\in W}\kl\left(\hmu_j^{T_i} + \Delta_W,\hmu_j^{T_i}+ \Delta_W+ 2^{-i}\Delta_{W}/120\right) \nonumber\\
&\leq c\log(t_i^2\Delta_W^{-1}/\delta)\max_{j \in W}\kl\left(\hmu_j^{T_i} ,\hmu_j^{T_i}+ 2^{-i}\Delta_{W}/120\right) \;. \label{eq:sD1}
\end{align}
By Definition of $\hDelta_{j,T_i}$ we also have $\forall z \leq \hDelta_{j,T_i}$, 
\begin{equation}\label{eq:sD2}
t_i\leq \log(t_i^2z^2/\delta)\max_{x \in W}\kl\left(\hmu_j^{T_i},\hmu_j^{T_i}+ z/120\right)        
\end{equation}

Thus via combination of Equations \eqref{eq:sD1} and \eqref{eq:sD2} the proof follows. 
\end{proof}
We are now ready to show that on our good event, we will eliminate all points within $W$. For technical reasons, we are required to condition on the intersection of events $\xi_{W,i-1}$ and $\xi_{W,i}$, to ensure all points are eliminated from $W$ at time $T_{i}$. 
\begin{proposition}\label{prop:xi}
For $i\in \mathbb{N}$, $W\subset [0,1]$, $x \in W \cap S_{T_i}$ with $j_x = \argmin_{j\in \cX_{T_i} \cap S_{T_i}}(|j - x|)$, we have that, on event $\xi_{W,i-1} \cap \xi_{W,i}$, 
$$ j_x \in \cQ_{T_i}\;.$$
\end{proposition}
\begin{proof}
Via Proposition \ref{prop:sizeD} we have that, on event $\xi_{W,i-1}$, for all $x \in W$, $\Delta_{(T_{i-1})} \leq \Delta_W/120  \leq \Delta(x)/120$. Thus  via the discretisation action of the algorithm, see line \ref{alg:disc} of \klcrank 
\begin{equation}\label{eq:disc}
  \max_{x,y\in \cX_{T_{i-1}} \cap S_{T_{i-1}}}|x-y| \leq (\Delta(x)/120)^{1/\beta}\;.  
\end{equation}
Now take $y \in U_{j_x,T_i}(6\Delta_{(T_i)})$ with $j_y = \argmin_{j\in \cX_{T_i} \cap S_{T_i}}(\lVert j - y\rVert)$, under event $\xi_{W,i}$, we have via Equation \eqref{eq:disc},
$$|y - j_y| \leq (\Delta(x)/120)^{1/\beta}\;,$$
and then via Assumption \ref{ass:1}, 
\begin{equation}\label{eq:etaD}
  |\eta(y) - \eta(j_y)| \leq \Delta(x)/120\;,  
\end{equation}
via definition of  $U_{j_x,T_i}(6\Delta_{(T_i)})$, on event $\xi_{W,i}$ we have 
\begin{equation}\label{eq:kappaD}
|\eta(j_x) - \eta(j_y)| \leq \Delta(x)/120\;,
\end{equation}
and thus via combination of Equations \eqref{eq:etaD} and \eqref{eq:kappaD},
$$|\eta(j_x) - \eta(y)| \leq \Delta(x)/120 + \Delta(x)/120 \leq \Delta(x)/60\;,$$ 
giving,
\begin{equation}\label{eq:uleq}
U_{j_x,T_i}(6\Delta_{(T_i)}) \subset \{y : |\eta(j_x) - \eta(y)| \leq \Delta(x)/60\}\subset  \{y : |\eta(x) - \eta(y)| \leq \Delta(x)/30\}\;.
\end{equation} 
Now, via definition of $\Delta(x)$, 
\begin{equation}\label{eq:maxU}
   \lambda\left(\{y : |\eta(x) - \eta(y)| \leq \Delta(x)\}\right) \leq\frac{p\epsilon(1 - \eta(x))}{\Delta(x)}\;,
\end{equation}
and so, via the fact that $\Delta_{(T_i)} \leq \Delta(x) /120$ and combination of Equations \eqref{eq:uleq} and \eqref{eq:maxU},

\begin{align*}
\Delta_{(T_i)} \leq 3\Delta(x)/160 &\leq \frac{3\epsilon p (1 - \eta(x))}{ 160\lambda(U_{j_x,T_i}(6\Delta_{(T_i)}))}\\
&\leq \frac{3\epsilon p ((1 - \hmu^t_{j_x})+ \Delta_{(T_i)})} { 80\lambda(U_{j_x,T_i}(6\Delta_{(T_i)}))}  \\
&\leq  \left(\frac{\epsilon \hat p_{T_i}}{\lambda(U_{j_x,T_i}(6\Delta_{(T_i)}))}\wedge 1\right)  (1 -\hmu^t_{j_x})\;, 
\end{align*}
where the third inequality comes from the fact that on event $\xi_{W,i}$, $|\hp_{T_i} - p|\leq p/2$.

\end{proof}

For the expected number of samples on a given subset $W$, we can expect our upper bound to depend upon the size of $W$, however, this relationship is not completely straightforward. The reason being, that the algorithm maintains a set of active points $\{\cX_t \cap S_t\}$, from which it samples. A disjoint set $W$ can be arbitrarily small but contain a large number of points in the active set $\{\cX_t \cap S_t\}$. Thus we introduce the following, 
For $W \in [0,1]$, let, $\tilde n = \argmin\left(n : 2^{-n} \leq c\Delta_{W}^{-1/\beta}\right)$ and consider the set of cells,
\begin{equation}\label{eq:barW}
\bar W := \bigcup_{x \in \cE_1(2^{-\tilde n}): \exists z \in W : |x-z|\leq 2^{-\tilde n} }\{y : |x-y| \leq 2^{-\tilde n} \}\;.    
\end{equation}
The set $\bar W$ is essentially the smallest possible extension of $W$, to a set comprised of the union of cells, on the dyadic grid of level smaller than $\Delta_{W}^{-1/\beta}$. 
\begin{proposition}\label{prop:barW}
For $W \in [0,1]$, such that $\forall x \in W$, $i\in \mathbb{N}$, on the event $\xi_{W,i-1} \cap \xi_{W,i}$ we have that 

$$N(W) \leq \lambda(\bar W)2^{i}\Delta_W^{-d/\beta} t_{i}\;,$$
where $N(W)$ is the total number of samples drawn by the learner on $W$. 
\end{proposition}
\begin{proof}
Via Proposition \ref{prop:xi} we have that on event $\xi_{W,i-1} \cap \xi_{W,i}$, 
$$\forall j \in \{W\cap\cX_{T_{i}}\}, j \in \cQ_{(T_{i})}\;,$$
as a result we have that for all 
$$t > T_{i}, W\cap S_{t} = \emptyset\;.$$ 
Thus,
$$N(W) \leq t_i |\{\cX_{T_i} \cap W\}|\;.$$
From Proposition \ref{prop:sizeD} we have that on the event $\xi_{W,i-1} \cap \xi_{W,i}$, $\Delta_{(T_i)} \geq c\Delta_{W}^{-1/\beta}2^{-i}$, and therefore, each cell of the dyadic grid of level smaller than $c\Delta_{W}^{-1/\beta}2^{-i}$ contains at most one point of $\{\cX_{T_i}\cap S_t\}$. Thus, the set $\bar W$ is a union of cells, each of volume $\Delta_W^{-1/{\beta}}$ and containing  at most one point of the set $\{\cX_{T_i} \cap W\}$, thus,
$$|\{\cX_{T_i} \cap W\}| \leq \lambda(\bar W) \Delta_W^{-1/\beta}\;,$$
and the result follows. 

\end{proof}
To bound in expectation the expected sampling time of the algorithm, on a given subset $W$, we will make use of the following equality $\E[N(W)] \leq \int_{z=0}^\infty\P\left(N(W)\geq z\right)\; dz$. A key step of our proof will be to replace the aforementioned integral, with a summation, $c\sum_{i=0}^\infty \P\left(N(W)\geq t_i\right)$. However, to do this we need to ensure the $t_i$ grow at most geometrically. We show just that in the following Proposition. 
\begin{proposition}\label{prop:ti}
For all $i\in \mathbb{N}$, we have that $4t_i \geq t_{i+1}$.
\end{proposition}
\begin{proof}
For all $x \in W$, via definition of $\Delta(x)$, we have that, $\Delta(x)\leq (1-\eta(x))$, thus,
\begin{equation}\label{eq:klD}
 \forall x \in W, \eta(x) +  120\Delta_W \leq 1 \;.  
\end{equation}
From Equation \eqref{eq:klD}, we have the following,
\begin{equation}\label{eq:kl0}
\max_{x\in W} \kl(\eta(x),\eta(x) + 2^{-i}\Delta_W /120) \leq   \kl(\eta(x),\eta(x) + 2^{-i-1}\Delta_W /120)/2\;,
\end{equation}
Now assume $t_{i+1} \geq 4 t_i$, under this assumption we would have,
$$4 t_{i} \leq \frac{\log(16t_i^2\Delta_W^{-a/\beta}/\delta)}{\kl(\eta(x),\eta(x) + 2^{-i-1}\Delta_W /120)} \leq  \frac{\log(t_i^2\Delta_W^{-a/\beta}/\delta)}{4\kl(\eta(x),\eta(x) + 2^{-i}\Delta_W /120)}$$
where the second inequality comes from application of Equation \eqref{eq:kl0}. Which is a contradiction, according to the definition of $t_i$.
\end{proof}

We are now ready to upper bound the expected number of samples on a given subset $W$.
\begin{proposition}\label{prop:Wbound}
For $W\subset [0,1]$, such that $\forall x \in W$, let $N(W)$ denote the total number of samples drawn on $W$. We have the following upper bound on, $N(W)$,
$$\frac{\E[N(W)]}{\lambda(\bar W)}\leq c\Delta^{-1/\beta}_{W}\max_{x\in W}  H(x)\log(c'  H(x)/\delta) \;.$$

\end{proposition}
\begin{proof}

\begin{align}
\frac{\E[N(W)]}{ \lambda(\bar W)} &\leq \int_{z=0}^\infty\P\left(N(W)/ \lambda(\bar W)\geq z\right)\; dz\\
&\leq t_0 \Delta^{-1/\beta}_{W} +  \sum_{i=1}^\infty \int_{z=2^i t_i\Delta_{W}^{-1/\beta}}^{2^{i+1} t_{i+1}\Delta_{W}^{-1/\beta}}\P\left(N(W)/ \lambda(\bar W) \geq z\right)\; dz \\
&\leq t_0 \Delta^{-1/\beta}_{W} +   c\sum_{i=1}^\infty t_i 2^i\Delta_{W}^{-1/\beta}\P\left(N(W)/ \lambda(\bar W)\geq 2^i t_i\Delta_{W}^{-1/\beta}\right)\label{eq:ti}\\
&\leq´t_0 \Delta^{-1/\beta}_{W} +  c\Delta_{W}^{-1/\beta}\sum_{i=1}^\infty t_i 2^i\P\left(\xi_{W,i-1}^c \cup \xi_{W,i}^c\right)\label{eq:pxi}\\
&\leq´t_0 \Delta^{-1/\beta}_{W} +  c\Delta_{W}^{-1/\beta}\sum_{i=1}^\infty t_i 2^i\P\left(\xi_{W,i}^c \right)\label{eq:fin}
\end{align}
Where the inequality of Equation \eqref{eq:ti} follows from Proposition \ref{prop:barW} and the inequality of Equation \eqref{eq:pxi} follows from Proposition \ref{prop:ti}. We now need to upper bound $\Delta_{W}^{-1/\beta}\sum_{i=1}^\infty t_i 2^i\P\left(\xi_{W,t_i}^c\right)$. Via Proposition \ref{prop:probxi} we have, 
\begin{equation}\label{eq:sumb}
 \Delta_{W}^{-1/\beta}\sum_{i=1}^\infty t_i 2^i\P\left(\xi_{W,i}^c\right) \leq \Delta_{W}^{-1/\beta}\sum_{t=1}^\infty t_i 2^i 2^{-i}t_i^{-3}\leq c \Delta^{-1/\beta}_{W}\;.
\end{equation}
Combination of Equations \eqref{eq:sumb} and \eqref{eq:fin} then leads to,
$$\frac{\E[N(W)]}{\lambda(\bar W)}\leq ct_0 \Delta^{-1/\beta}_{W}\;.$$
It now remains to upper bound $t_0$, set
$H_{W} = \max_{x \in W}\kl(\eta(x),\eta(x) +\Delta_{W})$,
$$ t_0 \leq  H_{W} \log(c'  H_{W}/\delta \Delta_W) \leq \max_{x \in W} H(x)\log(c'  H(x)/\delta \Delta_W) $$
\end{proof}
Now for a given subset $W$, we have a upper bound on the expected number of samples, which depends upon the minimum gap and maximum sample complexity across $W$. What remains is to then divide our feature space $\cX$ into areas of similar sample complexity and gap. With this in mind, consider the sequence of sets $G_{1,1},G_{1,2},...$ where,

$$G_{n,k} =\left\{x : \Delta(x) \in [2^{-n},2^{-n-1}], H(x) \in[2^{-k},2^{-k-1}] \right\}\;.$$
The final step of the proof will be to apply Proposition \ref{prop:Wbound}, to achieve a tight bound on the expected sampling time on each of the subsets $G_{n,k}$, individually. We can then upper bound the total expected sampling time by a summation across all $n,k \in \mathbb{N}$. However, as the reader will recall from Propisiton \ref{prop:Wbound}, for a given subset $W$ our upper bound on the expected sampling time across $W$, depends upon the expanded subset $\bar W$, see equation \eqref{eq:barW}. Therefore, we must ensure that, for a given $n,k$, the sample complexity and gap across $\bar G_{n,k}$ do not differ to greatly from $2^{-n},2^{-k}$ respectively. This is shown in the following Propositions. 
\begin{proposition}\label{prop:lD}
For a subset $W\subset [0,1]$, for all $x\in \bar W$, $\Delta(x) \leq 3\Delta_W$
\end{proposition}
\begin{proof}
As $x\in \bar W$, we have that $\exists z \in W, |x - y| \leq \Delta^{-1/\beta}_W$, and via Assumption \ref{ass:1}, $|\eta(x) - \eta(z)| \leq \Delta_W$. Now, 
\begin{align*}
\lambda\left(\{y : |\eta(y) - \eta(z)| \leq \Delta_W\}\right) &\geq \frac{p\epsilon(1 - \eta(z))}{\Delta_W} \nonumber\\
&\geq \frac{p\epsilon(1 - \eta(x) + \Delta_W) }{2\Delta_W}\nonumber\\
&\geq \frac{p\epsilon(1 - \eta(x)) }{2\Delta_W}\;.
\end{align*}
Where the second inequality comes from the fact that $\Delta_W \leq 1-\eta(x)$. As, 
$$\{y : |\eta(y) - \eta(z)| \leq 2\Delta_W\} \subset \{y : |\eta(y) - \eta(x)| \leq 3\Delta_W\}\;.$$
we then have 
\begin{equation}\label{eq:diffD}
 \lambda(\{y : |\eta(y) - \eta(x)| \leq 3\Delta_W\}) \geq \frac{p\epsilon(1 - \eta(x)) }{2\Delta_W}   
\end{equation}

Thus from Equation \eqref{eq:diffD} we have that, 
$$\lambda\left(\{y : |\eta(y) - \eta(x)| \leq 3\Delta_W\}\right)\geq \frac{p\epsilon(1 - \eta(x)) }{2\Delta_W}\geq \frac{p\epsilon(1 - \eta(x)) }{3\Delta_W}\;,$$
and as via definition of $\Delta(x)$, we have that, 
\begin{equation}\label{eq:defD}
\forall w \leq \Delta(x),\;\lambda\left(\{y : |\eta(y) - \eta(x)| \leq w\}\right)\leq \frac{p\epsilon(1 - \eta(x)) }{w}\;,
\end{equation}
it follows that $\Delta(x) \leq 3\Delta_W$.

\end{proof}
\begin{proposition}\label{prop:hD}
For a subset $W\subset [0,1]$, for all $x\in \bar W$, $\Delta(x) \geq \Delta_W$.
\end{proposition}
\begin{proof}
As $x\in \bar W$, we have that $\exists z \in W, |x - y| \leq \Delta^{-1/\beta}_W$, and via Assumption \ref{ass:1}, $|\eta(x) - \eta(z)| \leq \Delta_W$. We now prove by contradiction, let us assume that $\Delta(x) \leq \Delta_W$, we then have, 
\begin{align}
\lambda\left(\{y : |\eta(y) - \eta(x)| \leq \Delta_W\}\right) &>\frac{p\epsilon(1 - \eta(x))}{\Delta_W} \nonumber\\
&> \frac{p\epsilon(1 - \eta(z) + \Delta_W) }{\Delta_W}\nonumber\\
&> \frac{p\epsilon(1 - \eta(z)) }{2\Delta_W}\;.\label{eq:diffD2}
\end{align}
Where the second inequality comes from the fact that $\Delta_W \leq 1-\eta(z)$. As,
$$\{y : |\eta(y) - \eta(x)| \leq \Delta_W\} \subset \{y : |\eta(y) - \eta(z)| \leq 2\Delta_W\} $$
From Equation \eqref{eq:diffD2} we have directly that, 
$$\lambda\left(\{y : |\eta(y) - \eta(z)| \leq 2\Delta_W\}\right)\geq \frac{p\epsilon(1 - \eta(z)) }{2\Delta_W}\;,$$
which is then a contradiction, via the definition of $\Delta(z)$ - see Equation \eqref{eq:defD} in the proof of Proposition \ref{prop:lD} as $\Delta(z) \geq 2\Delta_W$.
\end{proof}
\begin{proposition}\label{prop:bH}
Let $n,k \in \mathbb{N}$, for all $x \in \bar G_{n,k}$, 
$$2^{k-4}\leq H(x) \leq 2^{k+4}\;.$$

\end{proposition}
\begin{proof}
Let $W = G_{n,k}$,  $x \in \bar W$ we have that $\exists z \in W :  |x - z| \leq \Delta^{-1/\beta}_W$, and via Assumption \ref{ass:1}, $|\eta(x) - \eta(z)| \leq \Delta_W$. Now,
\begin{align}
\nonumber
\kl(\eta(x),\eta(x) + c\Delta(x)) &\leq \kl(\eta(z) + \Delta_W ,\eta(z) + \Delta_W + c\Delta(x))\\
&\leq 2\kl\left(\eta(z),\eta(z) + c\Delta(x)\right)\\
&\leq 2\kl\left(\eta(z),\eta(z) + 2c\Delta(z)\right)\\
&\leq 4\kl\left(\eta(z),\eta(z) + c\Delta(z)\right)\;.    
\end{align}
Similarly we have, 
$$\kl(\eta(z),\eta(z) + c\Delta(z)) \leq 4\kl(\eta(x),\eta(x) + c\Delta(x))\;. $$

\end{proof}
Combination of Propositions \ref{prop:bH}, \ref{prop:hD} and \ref{prop:lD} leads to the following Proposition. 
\begin{proposition}\label{prop:fin}
Let $n,k \in \mathbb{N}$, for all $x \in \bar G_{n,k}$, we have that,
$$x \in \bigcup_{i=n-1,j=k-1}^{i=n+1,j=k+1} G_{n,k}\;,$$
and thus, 
$$c\sum_{n,k\in \mathbb{N}}\max_{x\in \bar G_{n,k}}\Delta(x)^{-\beta}H(x) \leq \sum_{n,k\in \mathbb{N}}\max_{x\in G_{n,k}}\Delta(x)^{-\beta}H(x) \;,$$
for some absolute constant $c>0$. 
\end{proposition}

We are now ready to prove Lemma \ref{lem:exp}.
\begin{proof}[Proof of Lemma \ref{lem:exp}]
 Via combination of Propositions \ref{prop:fin} and \ref{prop:Wbound} we have that, 
$$\sum_{n,k=1}^\infty \E[N(G_{n,k})] \leq c \sum_{n,k\in \lambda(G_{n,k})\mathbb{N}}\max_{x\in G_{n,k}}\Delta(x)^{-\beta}H(x)\log(c'  H(x)/\delta)$$

It now remains to lower bound the integral 
$$\int_0^1 \frac{ \Delta(y)^{-\beta}\log(H(y))}{\kl(\eta(y) - \Delta(y), \eta(y) + \Delta(y))}\;dy\;,$$
with the following, 
$$c \sum_{n =1,k=1 }^\infty \lambda(G_{n,k})2^{-n\beta} 2^{-k} \log(c'  H_{W}/\delta)\;,$$
where $c,c'$ are absolute constants.   
\end{proof}

\section{Proof of lower bound}\label{app:lb}

\paragraph{Piece wise $\beta$-Holder regression functions}
The class of problems $\cB$ consists of all problems $\nu$ such that the regression function $\eta$ is $\beta$-Holder. We define a new class of problems $\tilde \cB$, for which $\eta$ is piece wise $\beta$-holder continuous, on some known ordered partition of the interval $[0,1]$. Specifically, the class $\tilde \cB$ consists of the problems, for which there exists a known $M$ sized ordered partition of the interval $\cP = \{C_1,...,C_M\}$, where, for each $m\in  [M]$ $\eta$ is $\beta$-Holder on $C_m$, and furthermore, for all $m,n \in [M] : m>n$, we have that, 
$$\forall x\in C_m, y \in C_n, \eta(x) > \eta(y)\;.$$
Again, the ordered partition $\cP$ is known by the learner. We will now show that problems of the class $\tilde \cB$ are strictly easier than that of $\cB$. As the class of problems $\cB$ allows for any feature space of the form $[a,b]^d$ for $a,b\in \mathbb{R}$ with $a<b$, we immediately have the following result. 

\begin{lemma}\label{lem:piece}
Let $\delta,\epsilon > 0$. If there exists a PAC$(\epsilon,\delta)$ strategy $\pi$ such that on all problems $\nu \in \cB$ the expected sampling time of strategy $\pi$ is upper bounded by, $c \int H_\nu(x) \;dx,$ for some constant $c > 0$, where $H_\nu (x)$ is the complexity of point $x$ on problem $\nu$, then we have that on all problems $\nu \in \tilde \cB$, the expected sampling time of $\pi$ is upper bounded by $c \int H_\nu(x)\;dx.$
\end{lemma}
\begin{proof}[Proof of Theorem \ref{thm:lb}]
The proof will now follow by application of a Fano type inequality on a well chosen set of problems. 
\paragraph{Step 1: Constructing our well chosen set of problems}
As is typical in lower bounds, for a given problem $\nu$, we will wish to find a set of alternate problems, carefully chosen such that the gaps $\Delta(x)$ and therefore complexity, are close on the alternate set to $\nu$. In our setting this is tricky, as the gaps $\Delta(x)$ are dependent on the shape of the regression function $\eta$. Modifying $\eta$ locally can potentially have a global effect on the gaps $\Delta(x)$. To overcome this problem we will consider a well chosen $M$ sized partition the interval $[0,1]$, comprised of the sets $D_0,D_1,...,D_M$. For each $m\in [M]$ there is a corresponding representative $U_m \in D_m$. The partition is chosen such that on each set $D_m$ the gaps $\Delta(x)$ do not vary too much from $\Delta(U_m)$. Furthermore, the representatives are chosen to be sufficiently far from one another in terms of $\Delta$. In this fashion we can essentially modify $\eta$ on each set $D_m$ without effecting the value of $\Delta$ on the other sets, and as such construct our alternate set of problems. 

We define a set of points $U_0,U_1,...$ recursively as follows,
$U_0 = \argmin_{x\in [0,1]}(\Delta(x))$, for $m \geq 0$ we then define,
$$U_{m+1} = \argmin_{x\in [0,1]}\{\Delta(x) : \forall j\leq m, \exists k \in [U_j,x] : |\eta(U_j) - \eta(k)| \geq 3\Delta(k) + 3\Delta(U_j) \}\;,$$
note that we adopt the convention that, for all $a,b \in [0,1]$, $[a,b] = [b,a]$, thus $[U_j,x]$ is well defined in the case where $x<U_j$. Let $M$ be the largest integer for which $U_M$ exists. Note that the sequence $\left(\Delta(U_m)
 \right)_{m>M}$ is monotonically increasing and furthermore, for all $x \in [0,1]$, 
\begin{equation}\label{eq:leq2}
|\{m \in [M]:  \nexists y \in [U_m,x] : |\eta(U_m) - \eta(y)| \leq 3\Delta(U_m) + 3\Delta(y)\}| \leq 2\;. 
\end{equation}
We then define the corresponding set of groups, $D_0,D_1,...,D_M$ as follows. For $x\in [0,1]$ let $m,n\in [M]$ be such that, $U_{m} \leq x \leq U_n$ then set $i^+ = m\vee n$ and $i^- = m\wedge n$. If  $\nexists k \in [x,U_{i^+}] : |\eta(U_{i^+}) - \eta(k)| \leq 3\Delta_k + 3\Delta(U_{i^+})$ then $x \in D_{i^+}$, otherwise $i \in D_{i^-}$.

\begin{proposition}\label{prop:biddelta}
For all $m\in [M]$ we have that, $\forall x \in D_m,  \Delta(U_m) \leq \Delta(x)$. 
\end{proposition}

\begin{proof}
{\bf Case:1 Assume} $x \geq U_m$ Let $n$ be such that $U_m\leq x \leq U_n$. Via Equation \eqref{eq:leq2}, we must have that,
$$\exists k \in [n,x] : |\eta(x) - \eta(k)| \leq 3\Delta(k) + 3 \Delta_n\;,$$ 
and therefore, if $\Delta(x) < \Delta(U_m)$, then, 
$$\argmin\{\Delta(i) : \forall j \leq m-1, \exists k \in [U_j,i] : |\eta(U_j) - \eta(k)| \geq 3\Delta_k + 3\Delta(U_j) \} \neq U_m $$
which is a contradiction via the definition of $U_m$. 
\paragraph{Case:2 Assume $x \leq U_m$} Let $n$ be such that $U_n\leq x \leq U_m$. Via Equation \eqref{eq:leq2}, we must have that,
\begin{equation}\label{eq:xm}
\nexists k \in [x,m] : |\eta(x) - \eta(k)| \leq 3\Delta(k) + 3 \Delta(m)\;.
\end{equation}
Now if we also have, 
\begin{equation}\label{eq:xn}
\nexists k \in [n,x] : |\eta(x) - \eta(k)| \leq 3\Delta(k) + 3 \Delta(n)\;,     
\end{equation}
then, via combination of Equations \eqref{eq:xm} and \eqref{eq:xn}, 
$$\nexists k \in [n,m] : |\eta(x) - \eta(k)| \leq 3\Delta(k) + 3 \Delta(m)\wedge\Delta(n)\;.$$
Without loss of generality assume $\Delta(U_m) \leq \Delta(U_n)$. In this case we have that $m < n$, and, 

$$\nexists k \in [U_m,U_n]: |\eta(U_m) - \eta(k)| \geq 3\Delta(k) + 3 \Delta(U_m)\;,$$
a contradiction via the definition of $U_n$.

\end{proof}

\begin{proposition}
For $m \in [M]$, set $W_m =\{x : |\eta(x) - \eta(U_m)|\leq 3\Delta_{U_m}\}$. We have that,
$$\lambda(D_m) \leq 21\lambda(W_m)\;.$$
\end{proposition}
\begin{proof}
Firstly, take $x \in D_m$ and consider $J_x := \{y : |\eta(x) - \eta(y)| \leq \Delta(x)\}$. 
\begin{align}
   \Delta(x) &\leq \frac{\epsilon p (1- \eta(x))}{\lambda(J_x)}  \nonumber\\
   &\leq \frac{7\epsilon p (1- \eta(x))}{\lambda(J_x)} - \frac{6\epsilon p\Delta(x) }{\lambda(J_x)} \nonumber\\
   &\leq  \frac{7\epsilon p\left( (1-\eta(U_m)) +3\Delta(x) +3\Delta(U_m)\right)}{\lambda(R_x)}- \frac{6\epsilon p\Delta(x) }{\lambda(R_x)} \nonumber\\
    &\leq  \frac{7\epsilon p(1-\eta(U_m)) }{\lambda(J_x)} + \frac{6\epsilon p\Delta(x) }{\lambda(J_x)}- \frac{6\epsilon p\Delta(x) }{\lambda(J_x)}\label{eq:Dsets}
\end{align}
where the second inequality comes from the fact that $\Delta(x) \leq 1 - \eta(x)$. From proposition \ref{prop:biddelta}, $\Delta(x) \geq \Delta(U_m)$ \\
Now consider the sets, $R_1,R_2,...$ where for $n \in \mathbb{N}$, 
$$R_n := \{x \in D_m :\eta(U_m) + (3 + 3\cdot2^{n})\Delta(U_m) \leq \eta(x) \leq  \eta(U_m) + (3 + 3\cdot2^{n+1})\Delta(U_m)  \}\;.$$
For $n \in \mathbb{N}$ we will upper bound the size of $R_n$. Note that $\forall x\in R_n$, $\Delta(x) \geq 2^n\Delta(U_m) $,
thus via Equation \eqref{eq:Dsets}, for all $x\in R_n$, $\lambda(J_x) \leq 7\cdot2^{-n}\lambda(W_m)$ and thus $\lambda(R_n) \leq 3\cdot7\cdot2^{-n}\lambda(W_m)$. The result now follows by summing over all $R_n$.

\end{proof}
The proof of the following proposition follows via the same argument as in the  proof of 
Proposition \ref{prop:biddelta}.
\begin{proposition}\label{prop:size0}
For all $m \in [M]$, 
$$ |W_m | \leq c\lambda\{x \in W_m  : |\eta(x)- \eta(U_m)| \leq \Delta(U_m)\} $$  
\end{proposition}

We are now ready to construct our set of problems.
For $m\in [M]$, define the function,
$$f_m(x) = \begin{cases}
\forall x : x \mod 8\Delta(U_m)^{\beta} \leq 4\Delta(U_m)^{\beta}, \qquad f(x) = x^{1/\beta},\\
\forall x : x \mod 8\Delta(U_m)^{\beta} > 4\Delta(U_m)^{\beta}, \qquad f(x) = -(x-1)^{1/\beta} +1\;.
\end{cases}$$
Now let $\cG_m = \lfloor\lambda(W_m)\Delta(U_m)^{-\beta}\rfloor$, and define the points, 
$$(G_{i,m})_{i\in [\cG_m]} = \left(\sum_{n=0 }^{m-1} \lambda(W_n) +  8i\Delta(U_m)^{\beta}\right)_{i \in [\cG_m]}\;.$$ 
Now consider a family of problems $\nu^Q$ indexed by $Q \in \{-1,1\}^{\sum_{m\in [M]} \cG_m}$ and for $m \in [M], i \in [\cG_m]$, let $Q_i^m = Q(\sum_{n=1}^{m-1} \cG_{n} + i)$, where the target function $\eta_Q$ corresponding to $\nu^Q$ is defined as follows, 
$$\eta_Q(x) =\sum_{m\in M} \sum_{i \in [\cG_m]}\Ind(x \in [G_{i,m},G_{i+1,m}))\eta(U_m)  Q_i^m f_m(x)\;.$$ 
We see that the peak of each bump
Essentially the function $\eta_Q$ can be split into $m$ segments, each one being a series of "bumps" and "dips", each one beginning and ending at $\eta(U_m)$. The $i$th feature being a bump if $Q_i =1$, and a dip if  $Q_i =1$. We see that for the $m$th segment the max of each bump is $\eta(U_m) + \Delta(U_m)$ and the minimum of each dip is $\eta(U_m) -\Delta(U_m)$. Further more, by our choice of $f_m(x)$, we see $\eta_Q$ is piecewise Hölder smooth. For $m \in[M]$ we define the set of bumps in the $m$th segment as  
$$W^{Q,+}_m = \bigcup_{i \in [\cG_m] : Q_i = 1} [G_{i,m},G_{i+1,m})\;,$$ 
and similarly the set of dips as, 
$$W^{Q,-}_m = \bigcup_{i \in [\cG_m] : Q_i = -1} [G_{i,m},G_{i+1,m})\;. $$

 The following Lemma shows that, for a problem $\nu \in \cB$, the gaps and complexity across our family problems, $\nu^Q$ indexed by $Q$, does not differ too much from $\nu$.

\begin{lemma}\label{lem:del}
Given, $\nu \in \cB$, write $\Delta_Q (x)$ for the gap of point $x$ on problem $\nu_Q$. We have that for all $x \in [0,1]$, $\Delta(x) \geq c\Delta_Q(x)$, where $c>0$ is an absolute constant. 

\end{lemma}
\begin{proof}
If $\nexists m : x \in W_m$, we have that $\eta_Q(x) = 0$ and thus $\Delta_Q(x) = 1$. Thus assume $\exists m : x \in W_m$. As $x \in W_m$ we have $\Delta(x) \geq \Delta(U_m)$. Via proposition \ref{prop:size0} and identical reasoning to that in Equation \eqref{eq:Dsets} we have that $\Delta_Q(x) \geq c\Delta(U_m)$ and the result follows. 

\end{proof}
\paragraph{Step 2: showing that one suffers $\epsilon$ regret on a well chosen event}

We remind the reader that we denote the scoring function outputted by the learner as $\hat \eta$. For a set $C \subset [0,1]$, define,
$$\kappa(C) := \frac{1}{\lambda(C)}\int_C \eta(x) \;,dx \;,$$
and define

$$z_m := \min\left(z : H_{{\heta}}(z) \geq \frac{\lambda\left(\bigcup_{n=1}^{m-1} W_n\cup W_{m}^{Q,+}\right)(1 - \kappa\left(\bigcup_{n=1}^{m-1} W_n\cup W_{m}^{Q,+}\right)}{(1-p)}\right)\;,$$

Where we remind the reader that we define, $H_{\heta}(t)=\mathbb{P}\left\{ \heta(x)\leq t \mid Y =  0  \right\}$. Let $\hat Z_m \subset [0,1]$ be the largest set such that, $\forall x\in \hat Z_m, y \notin \hat Z_m, \hs(x) \leq \hs(y)$, and, 

$$   \lambda(\hat Z_m)(1- \kappa (\hat Z_m)) \leq \lambda\left(\bigcup_{n=1}^{m-1} W_n\cup W_{m}^{Q,+}\right)(1 - \kappa\left(\bigcup_{n=1}^{m-1} W_n\cup W_{m}^{Q,+}\right)\;.$$
that is, we have that $\hat Z_m = \{x : \hs(x) \geq z_m\}$. Note that $\hat Z_m$ is not necessarily unique, in this case we choose an arbitrary such $\hat Z_m$. Furthermore define, 
$$\hat Z_m^0 = \left\{x \in \hat Z_m : x \in \bigcup_{n=1}^{m-1} W_n \right\}, \qquad \hat Z_m^1 = \left\{x \in \hat Z_m : x \in  W^{Q,+}_{m}\right\}\;,$$
and
$$\hat Z_m^2 = \left\{x \in \hat Z_m : x \in W_{m}^{Q,-} \cup \bigcup_{n=m+1}^M W_n \right\}\;.$$

Now, define the event,
$$\xi_{i,m} := \bigg\{\{x \in [G_{i,m},G_{i+1,m}) : \hs(x) > z_m\} \leq \frac{\Delta(U_m)^\beta}{2} \bigg\}\;.$$

and then the events,
 $$\cE_1^m := \left\{\sum_{i \in \cG_m : Q_i = 1} \mathbf{1}(\xi_{i,m}) \leq \frac{|\cG_m|}{4}\right\},\qquad \cE_0^m :=  \left\{\sum_{i \in \cG_m : Q_i = -1} \mathbf{1}(\xi_{i,m}) \geq \frac{3|\cG_m|}{4}\right\}\;.$$

And let $\hat D_m = \bigcup_{n=1}^{m-1} C_n \setminus \hat Z_m^0$. Note that under event $\cE_1^m$, we have 
\begin{equation}\label{eq:zevent}
\lambda\left(\hat Z_m^1\right) \leq \frac{3\lambda(W_m^{Q,+})}{4}\;.
\end{equation}
Also, via the definition of $\hat Z_m$, we have that,
\begin{equation}\label{eq:zineq}
  \lambda(\hat Z_m^1)(1-\kappa(W_{m}^{Q,+}))  + \lambda(\hat Z_m^2)(1-\kappa(\hat Z_m^2))  =   \lambda(W_{m}^{Q,+}) (1-\kappa(W_{m}^{Q,+})) + \lambda(\hat D_m)(1 - \kappa( \hat D_m))\;.    
\end{equation}
which leads to, 

\begin{equation}\label{eq:new}
 \lambda(\hat Z_m^1)\kappa(\hat Z_m^2) +\lambda(\hat Z_m^2)\kappa(W_{m}^{Q,+}) = \lambda(\hat Z_m^1) + \lambda(\hat Z_m^2) - \lambda(W_{m}^{Q,+}) (1-\kappa(W_{m}^{Q,+})) + \lambda(\hat D_m)(1 - \kappa( \hat D_m))\;.   
\end{equation}

and also in combination with equation \ref{eq:zevent}
\begin{equation}\label{eq:new2}
(\lambda(W_m^{Q,+})/4 + \lambda(\hD_m))(1-\kappa(W_m^{Q,+})) \geq \lambda(\hat Z_m^2)\kappa(1-\hat Z_m^2)\;,
\end{equation}
which also gives $\lambda(\hat Z_m^2) \geq \lambda(W_m^{Q,+})/4$.

To complete Step: 2 we now lower bound $d_\infty(\hs,\eta)$ on event $\cE_1^m$. Firstly note that,

\begin{align*}
\mathrm{ROC}\left(\frac{(1-\hat{Z}_m)\lambda(\hat Z_m)}{1-p}, \eta\right) &= \frac{\lambda\left(\bigcup_{n=1}^{m-1} W_n \cup W_{m}^{Q,+}\right)\kappa\left(\bigcup_{n=1}^{m-1} W_n\cup W_{m}^{Q,+}\right)}{p}\\
&= \frac{\lambda\left(\bigcup_{n=1}^{m-1} W_n \right)\kappa\left(\bigcup_{n=1}^{m-1} W_n\right)}{p} + \frac{\lambda\left(W_{m}^{Q,+}\right)\kappa\left( W_{m}^{Q,+}\right)}{p}
\end{align*}
 therefore, for a problem  $\nu^Q : \sum_{i\in[\cG_m]} Q_i^m \geq \cG_m\lambda(W_m)/2$, on event $\cE_1^m$,
 
 \begin{align*}d_\infty(\hs,\eta)&= 1/p\left( \lambda(\hat D_m)\kappa(\hat D_m) + \lambda\left( W_{m}^{Q,+}\right)\kappa\left(  W_{m}^{Q,+}\right)- \frac{\kappa( W_{m}^{Q,+} )\lambda(\hat Z^1_m)}{p} - \kappa(W_{m}^{Q,-})\lambda(\hat Z^2_m)\right)\\
 &\geq 1/p\left( \lambda(\hat D_m) +  \lambda(W_{m}^{Q,+}) -\lambda(\hat Z_m^1) - \lambda(\hat Z_m^2) \right)  \\
 &\geq 1/p\left(  \lambda(W_{m}^{Q,+})/4 +\lambda(\hat D_m)  - \lambda(\hat Z_m^2) \right)  \\
&\geq 1/p\left( \frac{\lambda(\hZ_m^2)(1-\kappa(\hZ_m^2)}{1-\kappa(\hD_m)} -\lambda(\hZ_m^2) \right)  \\
&\geq 1/p\left( \frac{  \lambda(\hZ_m^2)(\kappa(\hD_m)-\kappa(\hZ_m^2))}{4(1-\kappa(\hD_m)})\right)\\ 
&\geq 1/p\left( \frac{ \lambda(W_{m}^{Q,+})(\kappa(W_{m}^{Q,+})-\kappa(W_{m}^{Q,-}))}{1-\kappa(W_{m}^{Q,+})}\right)\geq  1/p\left( \frac{ \lambda(W_{m}^{Q,+})\Delta(U_m))}{4(1-\kappa(W_{m}^{Q,+}))}\right)\geq \epsilon
\end{align*}
where the first inequality follows from Equation \eqref{eq:new} the second from  \eqref{eq:zevent} and the third from \eqref{eq:new2}.

Let $\P_{Q}$ correspond to the probability under the distribution on all samples collected by strategy $\pi$ on problem $\nu^Q$. Thus, as we assume policy $\pi$ is PAC$(\delta,\epsilon)$, on all problems $\nu^Q$, we must have that, on all problems $\nu^Q : \sum_{i\in[\cG_m] } Q_i^m \geq \cG_m\lambda(W_m)/2$, 
\begin{equation}\label{eq:bigevdel1}
\P_{Q}(\cE_1^m) \leq \delta\;,
\end{equation}
Via similar reasoning we can show that on all problems  $\nu^Q : \sum_{i\in[\cG_m] } Q_i^m \leq \cG_m/2$, we must have that, 
\begin{equation}\label{eq:bigevdel0}
   \P_{Q}(\cE_0^m) \leq \delta\;.
\end{equation}

\paragraph{Step 4: bounding the probability of the sum of $\xi_{i,m}$}
Now, for $m\in [M]$, and $\nu^Q : \sum_{i\in[\cG_m] } Q_i^m \leq \cG_m/2$, via the Azuma hoeffding inequality applied to the martingale,
$$\sum_{i :[\cG_m] \in C_m, Q_i^m = 0}\left[\mathbf{1}(\xi_{i,m}) - \mathbb{P}_Q(\xi_{i,m})\right]\;,$$
we have that,

\begin{equation}\label{eq:mart}
\P_Q\left( \sum_{i \in[\cG_m] : Q_i^m = 0}\left[\mathbf{1}(\xi_{i,m}) - \mathbb{P}_Q(\xi_{i,m})\right] \geq\cG_m\log\left(\frac{1}{1-\delta}\right)\right)  \leq 1-\delta\;.   
\end{equation}
Thus via combination of Equations \eqref{eq:bigevdel0} and \eqref{eq:mart} we must have that, $ \forall Q : \sum_{i\in[\cG_m] } Q_i^m \leq \cG_m/2,$ 

\begin{equation}\label{eq:lbsumxik0}
\sum_{i \in [\cG_m] : Q_i^m = 0}\P_Q(\xi_{i,m}) \leq \frac{\cG_m}{4} -\cG_m\log\left(\frac{1}{1-\delta}\right) \leq \frac{3\cG_m}{8} \;,
\end{equation}
where the second inequality comes from our assumption $\delta \leq 1-\exp(-1/8)$. Via similar reasoning we also have that, $ \forall Q : \sum_{i\in[\cG_m] } Q_i^m \geq \cG_m\lambda(W_m)/2,$ 

\begin{equation}\label{eq:lbsumxik1}
\sum_{i \in [\cG_m] : Q_i^m = 0}\P_Q(\xi_{i,m}) \geq \frac{\cG_m}{2} - \cG_m\log\left(\frac{1}{1-\delta}\right) \geq \frac{5\cG_m}{8}\;.
\end{equation}

\paragraph{Step 5: applying a Fano type inequality}

We write $\tau_{m,i}$ for the total number of samples the learner draws on $[G_{i,m},G_{i+1,m})$ and $\tau_{(m)}$ for the total number of samples the learner draws on the set $W_m$. Let $Q^{(i)}$ be the transformation of $Q$ that flips the $i$th  coordinate,
\[
Q_a^{(k)} = \begin{cases}
Q_a \text{ If }a\neq k,\\
1 - Q_a\text{ If }a = k\,.
\end{cases}\] 
Consider the class of problems,
$$\mathfrak{Q}_0 = \left\{Q :\forall m \in [M], \sum_{i : [\cG_m]}Q_i^m = \frac{\cG_m}{2}-1\right\}\;,$$ 
and 
$$\mathfrak{Q} = \left\{Q :\forall m \in [M], \sum_{i : [\cG_m]}Q_i^m = \frac{\cG_m}{2}\right\}\;,$$ 
and fix $m\in[M]$, we see that, for all $Q \in \mathfrak{Q}_0, i: Q_i = 0$, there exists a unique $\tilde Q \in \mathfrak{Q}$ such that $\tilde Q^{(i)} = Q$, therefore, 
\begin{equation}\label{eq:fano}
\sum_{Q \in \mathfrak{Q}}\sum_{m\in [M]} \; \sum_{i \in [\cG_m] : Q_i  = 1} \P_{Q^{(i)}}(\xi_{i,m})  =\sum_{Q \in \mathfrak{Q}_0} \; \sum_{i \in [\cG_m] : Q_i  = 0} \P_{Q}(\xi_{i,m}) \leq \frac{3\cG_m}{8}\;,    
\end{equation}
where the final inequality follows from Equation \eqref{eq:lbsumxik0}.
Thus, via combination of Equations \eqref{eq:lbsumxik1} and \eqref{eq:fano}, and using the data processing inequality and the convexity of the relative entropy we have, 
\begin{align*}
&\kl\!\left(\underbrace{\frac{1}{|\mathfrak{Q}|}\sum_{Q \in \mathfrak{Q}}\sum_{m\in [M]} \frac{2}{\cG_m}\; \sum_{i \in \cG_m : Q_i  = 1}  \P_{Q^{(i)}}(\xi_{i,m})}_{\leq 6/8} ,\underbrace{\frac{1}{|\mathfrak{Q}|}\sum_{Q \in \mathfrak{Q}} \frac{2}{\cG_m}\; \sum_{i \in \cG_m : Q_i  = 1}  \P_{Q}(\xi_{i,m})}_{\geq 10/8}\right)\\
&\leq \frac{1}{|\mathfrak{Q}|}\sum_{Q \in \mathfrak{Q}} \; \frac{2}{\cG_m}\sum_{i \in \cG_m : Q_i  = 1} \E_Q[\tau_{m,i}] \frac{\kl(\eta(U_m)  - 4/3\Delta(U_m),\eta(U_m) + 4/3\Delta(U_m))}{2}\\
&\leq \frac{1}{|\mathfrak{Q}|}\sum_{Q \in \mathfrak{Q}} \; \frac{\E_Q[\tau_{(m)}]\kl(\eta(U_m)  - 4/3\Delta_{U_m},\eta(U_m)  + 4/3\Delta(U_m))}{\cG_m}\;.
\end{align*}
Then using the Pinsker inequality $\kl(x,y)\geq 2(x-y)^2$, we obtain
\begin{equation*}
\frac{10}{8}\leq \frac{3}{8} + \sqrt{\leq \frac{1}{|\mathfrak{Q}|}\sum_{Q \in \mathfrak{Q}}  \; \frac{\E_Q[\tau_{(m)}]\kl(\eta(U_m) - 4/3\Delta(U_m),\eta(U_m) + 4/3\Delta(U_m))}{\cG_m}}\;,
\end{equation*}
and therefore, 

$$\leq \frac{1}{|\mathfrak{Q}|}\sum_{Q \in \mathfrak{Q}} \; \frac{\E_Q[\tau_{(m)}]\kl(\eta(U_m)
 - \Delta(U_m),\eta(U_m) + \Delta(U_m))}{\cG_m} \geq \frac{9}{64}$$
thus

\begin{align}
\leq \frac{1}{|\mathfrak{Q}|}\sum_{Q \in \mathfrak{Q}} \E_Q[\tau_{(m)}] &\geq \; c'\frac{\cG_m}{\kl(\eta(U_m) - 4/3\Delta(U_m),\eta(U_m) + 4/3\Delta(U_m))}\\
&\geq \; c'\frac{\lambda(W_m) \Delta(U_m)^{\beta}}{\kl(\eta(U_m) - 4/3\Delta(U_m),\eta(U_m) + 4/3\Delta(U_m))}
\end{align}
and thus, 
$$\max_{Q\in \mathfrak{Q}}\sum_{m\in [M]}\E_Q[\tau_{(m)}] \geq \; c'\sum_{m\in [M]}\frac{\lambda(W_m) \Delta(U_m)^{\beta}}{\kl(\eta(U_m) - 4/3\Delta(U_m),\eta(U_m) + 4/3\Delta(U_m))}$$
where $c' > 0$ is an absolute constant changing line by line. The proof now follows, as $\forall m\in [M], x\in W_m$,
$$H(x)\geq \frac{ \Delta(U_m)^{\beta}}{\kl(\eta(U_m) - 4/3\Delta(U_m),\eta(U_m) + 4/3\Delta(U_m))}\;.$$
\end{proof}

\end{document}